\documentclass[conference]{IEEEtran}
\makeatletter
\def\ps@headings{%
\def\@oddhead{\mbox{}\scriptsize\rightmark \hfil \thepage}%
\def\@evenhead{\scriptsize\thepage \hfil \leftmark\mbox{}}%
\def\@oddfoot{}%
\def\@evenfoot{}}
\makeatother
\IEEEoverridecommandlockouts
\usepackage{cite}
\usepackage{amsmath,amssymb,amsfonts, amsthm}
\usepackage[ruled, linesnumbered]{algorithm2e}
\usepackage{graphicx}
\usepackage{textcomp}
\usepackage{paralist}
\usepackage{adjustbox}
\usepackage{tabularray}

\newtheorem{myAttack}{Attack}

\newtheorem{proposition}{Proposition}
\usepackage{hyperref}

\def\BibTeX{{\rm B\kern-.05em{\sc i\kern-.025em b}\kern-.08em
    T\kern-.1667em\lower.7ex\hbox{E}\kern-.125emX}}

\usepackage{color, colortbl}
\definecolor{Gray}{gray}{0.9}
\definecolor{airforceblue}{rgb}{0.36, 0.54, 0.66}
\definecolor{aliceblue}{rgb}{0.94, 0.97, 1.0}
\definecolor{alizarin}{rgb}{0.82, 0.1, 0.26}
\definecolor{amber}{rgb}{1.0, 0.75, 0.0}
\definecolor{amber(sae/ece)}{rgb}{1.0, 0.49, 0.0}
\definecolor{bronze}{rgb}{0.8, 0.5, 0.2}
\definecolor{battleshipgrey}{rgb}{0.52, 0.52, 0.51}
\definecolor{bole}{rgb}{0.47, 0.27, 0.23}
\definecolor{bulgarianrose}{rgb}{0.28, 0.02, 0.03}
\definecolor{cadet}{rgb}{0.33, 0.41, 0.47}
\definecolor{ceil}{rgb}{0.57, 0.63, 0.81}
\definecolor{cerulean}{rgb}{0.0, 0.48, 0.65}
\definecolor{charcoal}{rgb}{0.21, 0.27, 0.31}
\definecolor{coolblack}{rgb}{0.0, 0.18, 0.39}
\definecolor{coolgrey}{rgb}{0.55, 0.57, 0.67}
\definecolor{darkcandyapplered}{rgb}{0.64, 0.0, 0.0}
\definecolor{darkbrown}{rgb}{0.4, 0.26, 0.13}
\definecolor{darkcerulean}{rgb}{0.03, 0.27, 0.49}
\definecolor{darkgray}{rgb}{0.66, 0.66, 0.66}
\definecolor{darkjunglegreen}{rgb}{0.1, 0.14, 0.13}
\definecolor{darktaupe}{rgb}{0.28, 0.24, 0.2}
\definecolor{frenchblue}{rgb}{0.0, 0.45, 0.73}
\definecolor{almond}{rgb}{0.94, 0.87, 0.8}
\definecolor{beaublue}{rgb}{0.74, 0.83, 0.9}
\definecolor{beige}{rgb}{0.96, 0.96, 0.86}
\definecolor{bisque}{rgb}{1.0, 0.89, 0.77}
\definecolor{black}{rgb}{0.0, 0.0, 0.0}
\definecolor{fluorescentorange}{rgb}{1.0, 0.75, 0.0}
\definecolor{ghostwhite}{rgb}{0.97, 0.97, 1.0}
\definecolor{antiquewhite}{rgb}{0.98, 0.92, 0.84}
\definecolor{LightCyan}{rgb}{0.88,1,1}
\newcommand{\hl}[1]{\colorbox{gray!40}{#1}}
\newcolumntype{a}{>{\columncolor{gray}}l}
\usepackage{placeins}
\hypersetup{
 hypertexnames=true, linkcolor=frenchblue, anchorcolor=black,
 citecolor=magenta, urlcolor=darkbrown  
}
\usepackage{stmaryrd}
\usepackage{breqn}
\usepackage{multicol}
\usepackage[skins]{tcolorbox}
\usepackage[nameinlink,noabbrev]{cleveref}
\newtcolorbox{myframe}[2][]{%
	enhanced,colback=white,colframe=black,coltitle=black,
	sharp corners,boxrule=0.6pt,
	fonttitle=\itshape,
	attach boxed title to top left={yshift=-0.3\baselineskip-0.4pt,xshift=2mm},
	boxed title style={tile,size=minimal,left=0.5mm,right=0.5mm,
		colback=white,before upper=\strut},
	title=#2,#1
}
\newtheorem{definition}{Definition}

\usepackage{caption}
\usepackage{subcaption}
\newcommand{\nick}{\texttt{SplitHappens}}
\begin{document}

\title{Split Happens: Combating Advanced Threats with Split Learning and Function Secret Sharing
% {
% \footnotesize \textsuperscript{*}Note: Sub-titles are not captured in Xplore and
 \thanks{This work was funded by the HARPOCRATES EU research project (No. 101069535).}
}

\author{\IEEEauthorblockN{1\textsuperscript{st} Tanveer khan}
\IEEEauthorblockA{\textit{Department of Computing Sciences} \\
\textit{Tampere University}\\
Tampere, Finland \\
tanveer.khan@tuni.fi}
\and
\IEEEauthorblockN{2\textsuperscript{nd} Mindaugas Budzys}
\IEEEauthorblockA{\textit{Department of Computing Sciences} \\
\textit{Tampere University}\\
Tampere, Finland \\
mindaugas.budzys@tuni.fi}
\and
\IEEEauthorblockN{3\textsuperscript{rd} Antonis Michalas}
\IEEEauthorblockA{\textit{Department of Computing Sciences} \\
\textit{Tampere University}\\
Tampere, Finland \\
antonios.michalas@tuni.fi}
% \and
% \and
% \and
}

\maketitle

\begin{abstract}
Split Learning (SL) -- splits a model into two distinct parts to help protect client data while enhancing Machine Learning (ML) processes. Though promising, SL has proven vulnerable to different attacks, thus raising concerns about how effective it may be in terms of data privacy. Recent works have shown promising results for securing SL through the use of a novel paradigm, named Function Secret Sharing (FSS), in which servers obtain shares of a function they compute and operate on a public input hidden with a random mask. However, these works fall short in addressing the rising number of attacks which exist on SL. In \nick, we expand the combination of FSS and SL to U-shaped SL. Similarly to other works, we are able to make use of the benefits of SL by reducing the communication and computational costs of FSS. However, a U-shaped SL provides a higher security guarantee than previous works, allowing a client to keep the labels of the training data secret, without having to share them with the server. Through this, we are able to generalize the security analysis of previous works and expand it to different attack vectors, such as modern model inversion attacks as well as label inference attacks. We tested our approach for two different convolutional neural networks on different datasets. These experiments show the effectiveness of our approach in reducing the training time as well as the communication costs when compared to simply using FSS while matching prior accuracy.
\end{abstract}


\section{Introduction}
\label{sec:introduction}

% Throughout the past decades, Machine Learning (ML) has garnered a %a%
% positive reputation for its applicability in various fields.
% It has, therefore, been actively used in applications ranging from image segmentation~\cite{oktay2018attention} to virtual assistants~\cite{kepuska2018next}. As a result, %This has led to%
% companies that offer Machine-Learning-as-a-Service (MLaaS) %to% 
% allow the outsourcing of computationally expensive ML tasks to a cloud server and make ML more accessible to general users. However, MLaaS raises privacy concerns as training the ML model requires users to share their private data with a potentially untrustworthy third-party server~\cite{shokri2017membership,hu2022membership}. Due to these concerns, Privacy-Preserving Machine Learning (PPML) has emerged as a novel research field to help protect private data in ML applications. 

The most widely used techniques for %providing%
Privacy-preserving Machine Learning (PPML) are cryptographic techniques, such as Homomorphic Encryption (HE)~\cite{tian2022sphinx,hesamifard2017cryptodl,sav2021poseidon,khan2024wildest, khan2023learning, nguyen2024pervasive, khan2023love, khan2023more, frimpong2024guardml} and Multi-Party Computation (MPC)~\cite{wagh2019securenn,ryffel2020ariann,wagh2021falcon}. HE %is a Public Key Encryption (PKE) technique, which 
allows users to perform mathematical computations, such as addition or multiplication, on encrypted data, 
while MPC aims to create methods that enable different parties to jointly compute a function on their private data, %while 
hence preserving privacy. Both techniques have been shown to improve the privacy of sensitive data while retaining high model accuracy~\cite{tian2022sphinx,wagh2022pika}. However, privacy comes at a cost and these techniques pertain to considerable computation and communication costs, that render their implementation 
in real-world scenarios difficult~\cite{nguyen2023split}.  
Consequently, researchers have been working to improve these techniques by introducing novel, cost-reducing approaches.
One of these solutions is the use of Function Secret Sharing (FSS)~\cite{boyle2015function}. Through the use of FSS multiple parties can perform computations on a public input using a private, secretly shared function without identifying the data or the function itself. This approach greatly reduces communication costs when compared to traditional MPC techniques, such as Additive Secret Sharing or Garbled Circuits~\cite{ryffel2020ariann,wagh2022pika,jawalkar2023orca}. Despite these improvements, the costs of performing full FSS training remain impractical. 

Because of %the %
high costs, there have been recent strides in the application of %applying%
Split Learning (SL). %a%
This collaborative learning %technique%
method has been applied to cryptographic techniques to both reduce the computational and communication costs of cryptographic techniques and address the privacy leakages inherently prevalent in SL~\cite{abuadbba2020can}. Researchers have shown that using SL and HE can address the privacy issue of SL~\cite{nguyen2023split}. %~\cite{khan2023split,nguyen2023split}. 
However, the main bottleneck is the high computational cost due to HE. To overcome this issue, 
many researchers use FSS instead of HE. Make Split not Hijack (MSnH)~\cite{khan2024make} was the first to combine FSS with Vanilla SL, significantly reducing computational costs compared to~\cite{nguyen2023split}. Another limitation of~\cite{nguyen2023split} is that the classification layer is executed on the server-side, inadvertently disclosing the model's final prediction.
%revealing the model's final prediction to the server.  %Researchers showcase that through the use of SL and HE~\cite{khan2023split,nguyen2023split}, both techniques can benefit from reduced computational costs as well as increased privacy.

Consequently, we aimed at applying %apply %We therefore aim to apply%
a U-shaped SL to FSS, for three reasons:  \begin{inparaenum}[\it (i)] \item This type of SL is more favorable as it allows protection for both data and labels \item Launching an attack in this setting is more challenging because the available resources for the adversary  %server 
present a variety of limitations %limited%
compared to vanilla SL protocol %so that it would be possible to further reduce the computational and communication costs for full private model training.
\item Additionally, using SL with FSS, we enhance the security of SL, by reducing the privacy leakages caused by modern SL attacks, such as Pseudo-Client ATtack (PCAT)~\cite{gao2023pcat}, Feature-Oriented Reconstruction Attack (FORA)~\cite{xu2024stealthy}, and Feature Sniffer~\cite{luo2023feature} and generalize the security for any Model Inversion Attacks (MIA) and Label Inference Attacks (LIA). 
% PCAT~\cite{gao2023pcat} and FORA~\cite{xu2024stealthy}.% and Label Sharing~\cite{khan2024make} %and Visual Invertibility (VI)~\cite{abuadbba2020can} 
\end{inparaenum}

\noindent \textit{\textbf{Contributions. }}
%The main contributions of this work, can be summarized as follows:
%As a result, %Following this,%
The main contributions of this work are:% follows:

\begin{enumerate}
    \item We designed \textit{\nick} -- an efficient PPML protocol 
    %\textit{Secret Shared Split Learning ($S^3L$)} 
    using SL and FSS.  %Our contribution lies in being the first to incorporate FSS into SL, aiming 
    %The goal of this work is to 
    \nick{} reduces the computation ~\cite{nguyen2023split} and communication costs of FSS~\cite{ryffel2020ariann}, while enhancing the privacy %aspect 
    of SL~\cite{khan2024make}. 
    \item We have conducted experiments on three distinct datasets, namely MNIST, CIFAR and FMNIST, % dataset to demonstrate privacy leakage in SL after the first and second convolution layers. 
    alongside two models. Initially, we present the privacy leakage stemming from the use of SL. Subsequently, we explain how the identified privacy concern in SL %are%
    is mitigated through the use of FSS. %We have used FSS to mitigate this issue, 
    Additionally, we demonstrate that \nick{} %our proposed protocol 
    provides a more efficient alternative to the computationally expensive HE method discussed in prior work by Nguyen \textit{et al.}~\cite{nguyen2023split}.  Our approach enables running multiple layers on the server-side, unlike~\cite{nguyen2023split}, which only considers %only%
    one server-side layer. % on the server-side.
    \item Our findings demonstrated that our approach exhibits comparable performance in terms of communication cost and complexity, while maintaining the same level of accuracy to MSnH through all tests while providing increased security.
    % \item We designed an FSS based vanilla SL (refer to as private vanilla SL model) protocol and compared it with AriaNN~\cite{ryffel2020ariann}. Our findings demonstrated that our approach exhibit comparable performance in terms of communication cost and complexity while maintaining the same level of accuracy. Additionally, our proposed protocol addresses the privacy concerns inherent in both FSS based vanilla SL and AriaNN.
    \item We expand security analysis of previous works, proving that \nick{} can negate PCAT, FORA and Feature Sniffer, and its security can be generalized to other MIAs and LIAs. %This shows our work provides a higher security than MSnH~\cite{khan2024make} and addresses advanced security issues within SL.
\end{enumerate}

\section{Related Work} 
\label{sec:related_Work}
Applying FSS in PPML has quickly become an active research field with a plethora of published works. AriaNN~\cite{ryffel2020ariann} is the first work that makes use of FSS for secure NN training and inference against a semi-honest adversary in a 2PC setting. AriaNN reports greatly reduced inference communication costs and inference time compared to other PPML MPC approaches, such as $ABY^3$~\cite{mohassel2018aby3} and SecureNN~\cite{wagh2019securenn}, only having higher communication costs than Falcon~\cite{wagh2021falcon} but lower computational complexity in the WAN setting.
Pika~\cite{wagh2022pika} introduces a 3PC PPML solution against either a semi-honest or a malicious adversary. The work shows lower communication costs than previous MPC works (i.e.\ $ABY^3$, Falcon), but has higher computation time than Falcon on higher batch sizes. 
Another recent FSS work is LLAMA~\cite{gupta2022llama}, which implements a low-latency library for computing non-linear mathematical functions, useful to PPML, such as sigmoid, tanh, reciprocal square root, using FSS primitives. 
Orca~\cite{jawalkar2023orca} makes use of and expands the library proposed in LLAMA~\cite{gupta2022llama} to perform secure training and inference on NN architectures with GPU acceleration.
FssNN~\cite{yang2023fssnn} proposes a two-party FSS framework for secure NN training and makes improvements to various aspects in prior arts. 
All of these works show much better performance in terms of training and inference time as well as lower communication costs when compared to previous arts in MPC. This shows that FSS can greatly improve current work in PPML. However, %the%
communication costs and computational complexity are still greatly increased %when compared%
comparing to plaintext alternatives. As a result, combining these works with other techniques, such as SL, could provide a substantial improvement to lowering %these%
said costs for NN training.

SL provides significant advantages in terms of splitting the computational cost of training NNs for a user by outsourcing a part of the training process to a server. Initially, it was believed as a promising approach in terms of client raw data protection, as %since%
parties exclusively share activation maps ($ATm$) %are shared%
between them. 
However, a potential vulnerability arises during end-to-end training, where the exchange of gradients at the cut layer could inadvertently encode private features or labels, posing a risk to data privacy. Abuadbba \textit{et al.}~\cite{abuadbba2020can} proposed using SL when training and classifying 1-dimensional data on CNN models. 
They identified that SL achieves comparable accuracy to centralized models, but noted that SL by itself has multiple privacy leakages 
hampering the confidentiality of the input data. Moreover, researchers have found that SL is susceptible to Feature Space Hijacking Attack (FSHA)~\cite{pasquini2021unleashing}, LIA~\cite{luo2023feature}, MIA~\cite{erdougan2022unsplit,gao2023pcat,xu2024stealthy} %, Exact attack~\cite{qiu2023exact} 
and Visual Invertibility (VI)~\cite{abuadbba2020can}.

The leakages highlight the importance of addressing data privacy concerns when implementing SL techniques. As such, researchers have attempted combining different Privacy-preserving Techniques (PPTs) with SL to address these attacks. For example, Abuadbba \textit{et al.}~\cite{abuadbba2020can} tried two privacy mitigation techniques--adding more hidden layers on the client-side and using Differential Privacy (DP). However, both %of%
these techniques suffer from a loss of model accuracy,
particularly when DP is used. Some other works like
Split Ways~\cite{khan2023split} or Split without a Leak~\cite{nguyen2023split} make use of HE to encrypt the $ATm$ of SL before outsourcing them to the server. %Through HE, 
This way, the model retains equivalent accuracy when compared to plaintext SL, and addresses some of the privacy leakages known in SL. 

In all of the covered FSS works, the researchers employ FSS on the entire model, which in turn increases computational complexity and communication overhead. To the best of our knowledge, 
only one work combines SL with FSS in the literature~\cite{khan2024make}. However, it has a limitation: the server can access labels and model output, compromising user data privacy.  
Our approach improves privacy in ML applications, solve the privacy leakage caused by SL and reduce both the communication and computation costs of FSS.

\section{Preliminaries}
\label{sec:prelim}

% This section covers the PPTs used in our design.

% \subsection{Split Learning}
% \label{subsec:splitlear}

\textbf{Split Learning:} SL is a collaborative learning technique~\cite{gupta2018distributed} in which an ML model is split into two parts. The client part ($f_{\theta_{C}}$), comprises the first few layers ($1, \ldots, l$) of the model, with $l$ being the last layer on the client-side. The server part ($f_{\theta_{P}}$), encompasses the remaining layers of the model ($l+1, \ldots , L$), with $L$ being the last layer on the server-side. The client and server collaborate to train the split model without having access to each other's parts.
%The client who owns the data uses forward propagation to train their part of the model and sends $ATm$ from the split layer (final layer of the client-side) to the server. The server continues the forward propagation on the $ATm$ using their part of the model. After completing the forward propagation and computing the loss, the server performs the backward propagation, only returning to the client the gradients to complete %the%
%backward propagation. This process is repeated until the model converges and learns a suitable set of parameters. %Although the client and server do not share any raw input data, this configuration requires label sharing. The sharing of labels can be eliminated by using a U-shaped SL configuration. The U-shaped SL configuration is nearly identical to the simple SL setup, with the exception that it does not require clients to share labels~\cite{vepakomma2018split}. On the server-side, the network is wrapped at the end layer, and the outputs are sent back to the client. Upon reception, client generates gradients from the end layers and utilize them for backward propagation without revealing the corresponding labels.

The aim of SL is to protect client privacy by allowing clients to train part of the model, and share $ATm$ (instead of %their%
raw data) with a server running the remaining part of the model. It is also utilized to reduce the client's computational overhead by merely running a few layers rather than the entire model. In the local model, there is no split, while in the vanilla SL there is a split with labels being shared with the server. In contrast, the U-shaped SL model also has a split though the final layer
is executed on the client-side, resulting in no sharing of labels with the server. Initially, it was believed that SL is a promising approach in terms of client raw data protection, as %since%
parties exclusively share $ATm$ %are shared%
between them. %parties.%
However, studies have shown the possibility of privacy leakage in SL~\cite{abuadbba2020can,vepakomma2019reducing}.

\textbf{Function Secret Sharing:} FSS provides a way for additively secret-sharing a function $f$ from a given function family $F$. A two-party FSS scheme splits a function $f: {(0,1)}^{n} \rightarrow \textit{G}$, for some abelian group $G$, into functions described by keys such that $f = \mathsf{f_{0}} + \mathsf{f}_{1}$ and every strict subset of the keys hides $f$. Formally an FSS scheme can be defined as~\cite{boyle2015function}:
\begin{definition}[Function Secret Sharing]
    Let $\mathsf{FSS}$ be a 2-party Function Secret Sharing scheme, that %An FSS scheme 
    for some class $F$ has a pair of Probabilistic Polynomial-Time (PPT) algorithms $\mathsf{FSS} = (\mathsf{KeyGen, EvalAll})$ such that:

\begin{itemize}
	\item $(\mathsf{f_{0}, f_{1}}) \leftarrow \mathsf{KeyGen(1}^{\lambda}, f)$: 
	The $\mathsf{KeyGen}$ algorithm is a probabilistic algorithm that takes as input the security parameter $\lambda$ and some function $f$ %to be secret shared 
    and outputs two different function shares $(\mathsf{f_{0}, f_{1}})$ (also called function keys).% $(k_{0}, k_{1})$).
	\item $\mathsf{EvalAll}(j, \mathsf{f}_{j}, x_{pub}) \rightarrow \mathsf{f}_{j}(x_{pub})$: The $\mathsf{EvalAll}$ algorithm is a deterministic algorithm that takes three parameters, the input bit $j \in \{0,1\}$, the function key $\mathsf{f}_j$, and the public data $x_{pub}$ and outputs the shares $\mathsf{f}_{j}(x_{pub})$. 
\end{itemize}
%\smallskip

\noindent An FSS scheme should satisfy the following two properties:

\begin{itemize}
	\item \underline{Secrecy}: A single key $(\mathsf{f}_{j})$ hides the original function $f$. 
	\item \underline{Correctness}: Adding the local shares gets the same output as the original function:
    \begin{equation}
        \mathsf{Pr[f_0}+\mathsf{f_1} = f | (\mathsf{f}_0, \mathsf{f}_1) \leftarrow \mathsf{KeyGen}(1^\lambda, f)] = 1
    \end{equation}
    %$(\mathsf{f_{0}}(x_{pub}) + \mathsf{f_{1}}(x_{pub}) = f(x_{pub}))$.
\end{itemize}
\end{definition}
% More details about the FSS primitives are in~\autoref{sec:appendix}.

% \section{Architecture}
% \label{sec:architecture}
% This section introduces the ML classification model, starting with the FSS-based %model from two related works, AriaNN~\cite{ryffel2020ariann} and 
% approach MSnH~\cite{khan2024make} (see \autoref{sec:related_Work}). %(due to space limitations preliminaries are available in \hyperref[app:preliminaries]{Appendix A}). 
% We then present \nick{} -- our U-shaped SL protocol, and outline the client and server roles in split model training. %. Finally, we describe key entities involved in training of split model, namely the client and the server, elaborating on their respective roles and assigned parameters during training phase.
% \autoref{table:paranddes} summarizes notation of this paper.
\begin{table*}
\scriptsize
    \centering
    \caption{Parameters and description in the algorithms}
    \resizebox{\textwidth}{!}{%
    \label{table:paranddes}
    \begin{tabular}{l|l|l|c|l|l}
        \hline
        \rowcolor{gray}
        \color{white}\textbf{\#}	& \color{white}\textbf{ML Parameters} & \color{white}\textbf{Description} &  & \color{white}\textbf{FSS Parameters} & \color{white}\textbf{Description}\\ 
        \hline
        1	& $\mathbf{D}$, $m$  & Dataset, Number of data samples & & $s_{j}$ & Random Seed \\ \hline
        2	& $x, y$ & Input data samples, Ground-truth labels & &  $\alpha$  & Random mask\\ \hline
        3	& $\eta$, $p$ & Learning rate, Momentum &  &  $P_{0}$, $P_{1}$ &  Server~1 \ and \ 2\\ \hline
	4	& $\Delta$ & Backpropogation gradients &  & $\mathsf{f_0, f_1}$ & Function shares (function keys) \\ \hline
        5	& $n$, $x_{pub}$ & Batch size, Public input  & & $f_{\theta_{P}}$  & Server-side model \\ \hline
	6	& $N$ & Number of batches to be trained & & $f_{\theta_{C}}$   & Client-side model \\ \hline
        7	& $E$ & Number of training epochs &  &   $\tilde{f}$, $\tilde{f}^{-1}$ & Encoder, Decoder \\ \hline
        8	& $f^{i}$ & Linear or non-linear operation of layer $i$ & & $\mathsf{f^{FC}_{j}}$ &  FC part for server $j$ \\ \hline
        9	& $ATm^{i}$ & Output activation map of $f^{i}$& & $\mathsf{f^{ReLU}_{j}}$  & ReLU FSS key for server $j$  \\ \hline
    \end{tabular}
    }
\end{table*}

\begin{figure}[t]
	\centering
\tikzset{every picture/.style={line width=0.65pt}} %set default line width to 0.75pt        
\begin{adjustbox}{width=0.4\textwidth}

\tikzset{every picture/.style={line width=0.75pt}} %set default line width to 0.75pt        

\begin{tikzpicture}[x=0.75pt,y=0.75pt,yscale=-1,xscale=1]

\draw  [fill={rgb, 255:red, 255; green, 244; blue, 199 }  ,fill opacity=1 ] (205,312.62) -- (217.9,299.72) -- (248,299.72) -- (248,421.82) -- (235.1,434.72) -- (205,434.72) -- cycle ; \draw   (248,299.72) -- (235.1,312.62) -- (205,312.62) ; \draw   (235.1,312.62) -- (235.1,434.72) ;
%Shape: Cube [id:dp3608175823822565] 
\draw  [fill={rgb, 255:red, 200; green, 218; blue, 164 }  ,fill opacity=1 ] (235,322.32) -- (247.6,309.72) -- (277,309.72) -- (277,412.22) -- (264.4,424.82) -- (235,424.82) -- cycle ; \draw   (277,309.72) -- (264.4,322.32) -- (235,322.32) ; \draw   (264.4,322.32) -- (264.4,424.82) ;
%Shape: Cube [id:dp05648282822730866] 
\draw  [fill={rgb, 255:red, 218; green, 246; blue, 242 }  ,fill opacity=1 ] (263.2,323.52) -- (277,309.72) -- (309.2,309.72) -- (309.2,411.92) -- (295.4,425.72) -- (263.2,425.72) -- cycle ; \draw   (309.2,309.72) -- (295.4,323.52) -- (263.2,323.52) ; \draw   (295.4,323.52) -- (295.4,425.72) ;
%Shape: Cube [id:dp9737918131319308] 
\draw  [fill={rgb, 255:red, 197; green, 181; blue, 175 }  ,fill opacity=1 ] (142,310.62) -- (154.9,297.72) -- (185,297.72) -- (185,440.82) -- (172.1,453.72) -- (142,453.72) -- cycle ; \draw   (185,297.72) -- (172.1,310.62) -- (142,310.62) ; \draw   (172.1,310.62) -- (172.1,453.72) ;
%Shape: Cube [id:dp27994533143945477] 
\draw  [fill={rgb, 255:red, 255; green, 244; blue, 199 }  ,fill opacity=1 ] (295.4,310.62) -- (308.3,297.72) -- (338.4,297.72) -- (338.4,418.82) -- (325.5,431.72) -- (295.4,431.72) -- cycle ; \draw   (338.4,297.72) -- (325.5,310.62) -- (295.4,310.62) ; \draw   (325.5,310.62) -- (325.5,431.72) ;
%Shape: Cube [id:dp6352544004193167] 
\draw  [fill={rgb, 255:red, 200; green, 218; blue, 164 }  ,fill opacity=1 ] (325,322.32) -- (337.6,309.72) -- (367,309.72) -- (367,412.12) -- (354.4,424.72) -- (325,424.72) -- cycle ; \draw   (367,309.72) -- (354.4,322.32) -- (325,322.32) ; \draw   (354.4,322.32) -- (354.4,424.72) ;
%Shape: Cube [id:dp936555164326225] 
\draw  [fill={rgb, 255:red, 218; green, 246; blue, 242 }  ,fill opacity=1 ] (354.4,320.7) -- (365.38,309.72) -- (391,309.72) -- (391,413.74) -- (380.02,424.72) -- (354.4,424.72) -- cycle ; \draw   (391,309.72) -- (380.02,320.7) -- (354.4,320.7) ; \draw   (380.02,320.7) -- (380.02,424.72) ;
%Shape: Cube [id:dp47363183361521477] 
\draw  [fill={rgb, 255:red, 246; green, 200; blue, 185 }  ,fill opacity=1 ] (348.6,451.62) -- (361.5,438.72) -- (391.6,438.72) -- (391.6,523.9) -- (378.7,536.8) -- (348.6,536.8) -- cycle ; \draw   (391.6,438.72) -- (378.7,451.62) -- (348.6,451.62) ; \draw   (378.7,451.62) -- (378.7,536.8) ;
%Straight Lines [id:da1757227418166537] 
\draw [line width=1.5]    (185,363.72) -- (201,363.72) ;
\draw [shift={(205,363.72)}, rotate = 180] [fill={rgb, 255:red, 0; green, 0; blue, 0 }  ][line width=0.08]  [draw opacity=0] (11.61,-5.58) -- (0,0) -- (11.61,5.58) -- cycle    ;
%Straight Lines [id:da4368542458698682] 
\draw [line width=1.5]    (398,380) -- (437,380) ;
\draw [shift={(441,380)}, rotate = 180] [fill={rgb, 255:red, 0; green, 0; blue, 0 }  ][line width=0.08]  [draw opacity=0] (11.61,-5.58) -- (0,0) -- (11.61,5.58) -- cycle    ;
%Straight Lines [id:da929708417882944] 
\draw [line width=1.5]    (514,493.66) -- (401,493.02) ;
\draw [shift={(397,493)}, rotate = 0] [fill={rgb, 255:red, 0; green, 0; blue, 0 }  ][line width=0.08]  [draw opacity=0] (11.61,-5.58) -- (0,0) -- (11.61,5.58) -- cycle    ;
%Straight Lines [id:da023090008773704485] 
\draw [line width=1.5]    (512.44,472) -- (513,493.66) ;
%Shape: Cube [id:dp27874834507241586] 
\draw  [fill={rgb, 255:red, 195; green, 220; blue, 252 }  ,fill opacity=1 ] (451,284.32) -- (460.6,274.72) -- (483,274.72) -- (483,333.4) -- (473.4,343) -- (451,343) -- cycle ; \draw   (483,274.72) -- (473.4,284.32) -- (451,284.32) ; \draw   (473.4,284.32) -- (473.4,343) ;
%Shape: Cube [id:dp9520531144676215] 
\draw  [fill={rgb, 255:red, 195; green, 220; blue, 252 }  ,fill opacity=1 ] (544,286.42) -- (554.2,276.22) -- (578,276.22) -- (578,332.8) -- (567.8,343) -- (544,343) -- cycle ; \draw   (578,276.22) -- (567.8,286.42) -- (544,286.42) ; \draw   (567.8,286.42) -- (567.8,343) ;
%Shape: Cube [id:dp5827608909417462] 
\draw  [fill={rgb, 255:red, 195; green, 220; blue, 252 }  ,fill opacity=1 ] (452,406.6) -- (461.6,397) -- (484,397) -- (484,456.12) -- (474.4,465.72) -- (452,465.72) -- cycle ; \draw   (484,397) -- (474.4,406.6) -- (452,406.6) ; \draw   (474.4,406.6) -- (474.4,465.72) ;
%Shape: Cube [id:dp010414880699535334] 
\draw  [fill={rgb, 255:red, 195; green, 220; blue, 252 }  ,fill opacity=1 ] (543,406.6) -- (552.6,397) -- (575,397) -- (575,455.62) -- (565.4,465.22) -- (543,465.22) -- cycle ; \draw   (575,397) -- (565.4,406.6) -- (543,406.6) ; \draw   (565.4,406.6) -- (565.4,465.22) ;
%Straight Lines [id:da9646095033679685] 
\draw    (484.5,432) -- (493,432) ;
\draw [shift={(496,432)}, rotate = 180] [fill={rgb, 255:red, 0; green, 0; blue, 0 }  ][line width=0.08]  [draw opacity=0] (8.93,-4.29) -- (0,0) -- (8.93,4.29) -- cycle    ;
%Straight Lines [id:da6203869251076094] 
\draw    (483,313) -- (493,313) ;
\draw [shift={(496,313)}, rotate = 180] [fill={rgb, 255:red, 0; green, 0; blue, 0 }  ][line width=0.08]  [draw opacity=0] (8.93,-4.29) -- (0,0) -- (8.93,4.29) -- cycle    ;
%Shape: Rectangle [id:dp3229000495351909] 
\draw  [dash pattern={on 0.84pt off 2.51pt}] (445,270) -- (582,270) -- (582,350) -- (445,350) -- cycle ;
%Shape: Rectangle [id:dp7922022812314923] 
\draw  [dash pattern={on 0.84pt off 2.51pt}] (447,393) -- (578,393) -- (578,472) -- (447,472) -- cycle ;
%Shape: Rectangle [id:dp28111587668663407] 
\draw  [dash pattern={on 0.84pt off 2.51pt}] (138,289) -- (397,289) -- (397,545) -- (138,545) -- cycle ;
%Straight Lines [id:da8569336362614827] 
\draw [line width=1.5]    (599.02,519) -- (401,518.02) ;
\draw [shift={(397,518)}, rotate = 0] [fill={rgb, 255:red, 0; green, 0; blue, 0 }  ][line width=0.08]  [draw opacity=0] (11.61,-5.58) -- (0,0) -- (11.61,5.58) -- cycle    ;
%Straight Lines [id:da9714785105630487] 
\draw [line width=1.5]    (597.02,318) -- (598.02,519) ;
%Straight Lines [id:da3794746692897176] 
\draw [line width=1.5]    (579,318) -- (598.02,318) ;
%Shape: Cube [id:dp648200515982708] 
\draw  [fill={rgb, 255:red, 246; green, 200; blue, 185}  ,fill opacity=1 ] (495,286.2) -- (505.2,276) -- (529,276) -- (529,332.8) -- (518.8,343) -- (495,343) -- cycle ; \draw   (529,276) -- (518.8,286.2) -- (495,286.2) ; \draw   (518.8,286.2) -- (518.8,343) ;
%Shape: Cube [id:dp07724851878760886] 
\draw  [fill={rgb, 255:red, 246; green, 200; blue, 185}  ,fill opacity=1 ] (495,408.2) -- (505.2,398) -- (529,398) -- (529,454.8) -- (518.8,465) -- (495,465) -- cycle ; \draw   (529,398) -- (518.8,408.2) -- (495,408.2) ; \draw   (518.8,408.2) -- (518.8,465) ;
%Straight Lines [id:da8636497282761513] 
\draw    (530,313) -- (540,313) ;
\draw [shift={(543,313)}, rotate = 180] [fill={rgb, 255:red, 0; green, 0; blue, 0 }  ][line width=0.08]  [draw opacity=0] (8.93,-4.29) -- (0,0) -- (8.93,4.29) -- cycle    ;
%Straight Lines [id:da44146299720339477] 
\draw    (529,432) -- (539,432) ;
\draw [shift={(542,432)}, rotate = 180] [fill={rgb, 255:red, 0; green, 0; blue, 0 }  ][line width=0.08]  [draw opacity=0] (8.93,-4.29) -- (0,0) -- (8.93,4.29) -- cycle    ;

\draw (210,418) node [anchor=north west][inner sep=0.75pt]  [rotate=90] [align=left] {Convolutional};

\draw (303,418) node [anchor=north west][inner sep=0.75pt]  [rotate=90] [align=left] {Convolutional};

\draw (240.42,415) node [anchor=north west][inner sep=0.75pt]  [rotate=90] [align=left] {Max-pooling};

\draw (333.42,415) node [anchor=north west][inner sep=0.75pt]  [rotate=90] [align=left] {Max-pooling};

\draw (272.42,390) node [anchor=north west][inner sep=0.75pt]  [rotate=90] [align=left] {ReLU};

\draw (360.42,390) node [anchor=north west][inner sep=0.75pt]  [rotate=90] [align=left] {ReLU};

\draw (164,476.72) node [anchor=north west][inner sep=0.75pt]   [align=left] {\textbf{Client}};

\draw (355.42,512.55) node [anchor=north west][inner sep=0.75pt]  [rotate=90] [align=left] {ReLU};

\draw (149,425) node [anchor=north west][inner sep=0.75pt]  [rotate=90] [align=left] {Training Data};

\draw (453,325) node [anchor=north west][inner sep=0.75pt]  [rotate=90]  {$\mathsf{f^{FC}_{0}}$};

\draw (547,325) node [anchor=north west][inner sep=0.75pt]  [rotate=90]  {$\mathsf{f^{FC}_{0}}$};

\draw (454,445) node [anchor=north west][inner sep=0.75pt]  [rotate=90]  {$\mathsf{f^{FC}_{1}}$};

\draw (545,445) node [anchor=north west][inner sep=0.75pt]  [rotate=90]  {$\mathsf{f^{FC}_{1}}$};

\draw (498,326) node [anchor=north west][inner sep=0.75pt]  [rotate=90]  {$\mathsf{f^{ReLU}_{0}}$};

\draw (498,448) node [anchor=north west][inner sep=0.75pt]  [rotate=90]  {$\mathsf{f^{ReLU}_{1}}$};

\draw (400,360) node [anchor=north west][inner sep=0.75pt]    {$x_{pub}$};

\draw (485,365) node [anchor=north west][inner sep=0.75pt]    {$\mathsf{EvalAll}$};

\draw (485,250) node [anchor=north west][inner sep=0.75pt]   [align=left] {\textbf{Server 1}};

\draw (525,475) node [anchor=north west][inner sep=0.75pt]   [align=left] {\textbf{Server 2}};

\draw (575,498) node [anchor=north west][inner sep=0.75pt]    {$\mathbf{\hat{y}}_{0}$};

\draw (419,466) node [anchor=north west][inner sep=0.75pt]    {$\mathbf{\hat{y}}_{1}$};

\draw (410,500) node [anchor=north west][inner sep=0.75pt]  {$+$};

\end{tikzpicture}
\end{adjustbox}
	\caption{\nick: Private U-shaped SL}
	\label{fig: ushapedsl}
\end{figure}
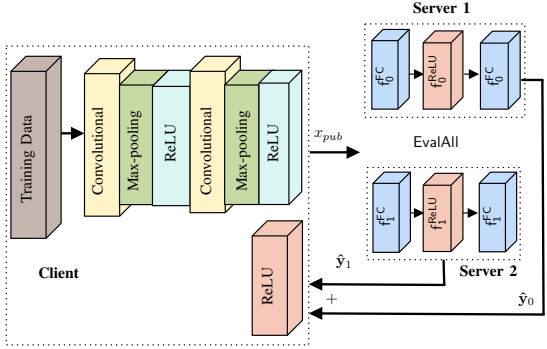

\section{FSS based U-shaped SL Protocol}
\label{sec:methodology}

This section presents the \nick{} protocol's training algorithms and system model.

\subsection{Actors in the model}
\label{subsec:slActors}

\begin{itemize}
    \item \textbf{Client}: The client provides the training data and is capable of computing the initial layer of the model in plaintext. The client uses their private data to compute the initial layers %of the model 
    and shares $ATm$ and the remaining layers with the servers. %In our %proposed 
    %approach, the client has to share the labels with the server. 
    In our setup, the client also computes the final prediction of the model \textit{independently} and does \textit{not share the truth labels to the server.}
    \item \textbf{Servers}: Our protocol requires two servers to perform computations on the secret shares. %The servers compute on the ML model shares using the received $x_{pub}$. 
    Based on the received $x_{pub}$, each server makes computations on the underlying ML model. Since we employ the U-shaped SL protocol, %Both servers can compute the final outputs and compare it to the truth labels received from the client, but in U-shaped SL, 
    the servers are unable to compute the model's final output %f the  model%
    as the servers do not receive the truth labels. In our approach, the final output of the model is computed on the client-side. 
    We operate under the assumption that there is no collusion between the two servers and that the servers are hosted independently to avoid reconstructing the original data from the shares. This assumption is realistic and consistent with other 2PC works~\cite{mohassel2017secureml,ryffel2020ariann}.
\end{itemize}

\begin{figure*}
    \centering
    \begin{adjustbox}{width=0.8\textwidth}
    \tikzset{every picture/.style={line width=0.75pt}} %set default line width to 0.75pt           

\begin{tikzpicture}[x=0.75pt,y=0.75pt,yscale=-1,xscale=1]
%uncomment if require: \path (0,555); %set diagram left start at 0, and has height of 555

%Shape: Rectangle [id:dp6361759885625646] 
\draw  [color={rgb, 255:red, 210; green, 82; blue, 70 }  ,draw opacity=1 ][fill={rgb, 255:red, 210; green, 82; blue, 70 }  ,fill opacity=0.3 ][line width=0.75]  (293.57,126.14) -- (311.57,126.14) -- (311.57,164.14) -- (293.57,164.14) -- cycle ;
%Shape: Circle [id:dp6410513080124076] 
\draw  [fill={rgb, 255:red, 255; green, 255; blue, 255 }  ,fill opacity=1 ] (307.29,153.29) .. controls (307.29,150.37) and (304.92,148) .. (302,148) .. controls (299.08,148) and (296.71,150.37) .. (296.71,153.29) .. controls (296.71,156.2) and (299.08,158.57) .. (302,158.57) .. controls (304.92,158.57) and (307.29,156.2) .. (307.29,153.29) -- cycle ;
%Shape: Circle [id:dp266392757190985] 
\draw  [fill={rgb, 255:red, 255; green, 255; blue, 255 }  ,fill opacity=1 ] (307.29,136.86) .. controls (307.29,133.86) and (304.86,131.43) .. (301.86,131.43) .. controls (298.86,131.43) and (296.43,133.86) .. (296.43,136.86) .. controls (296.43,139.86) and (298.86,142.29) .. (301.86,142.29) .. controls (304.86,142.29) and (307.29,139.86) .. (307.29,136.86) -- cycle ;

%Shape: Rectangle [id:dp8616360763535557] 
\draw  [fill={rgb, 255:red, 255; green, 244; blue, 199}  ,fill opacity=0.52 ] (498.34,104.32) -- (513.97,104.32) -- (513.97,141.36) -- (498.34,141.36) -- cycle ;
%Shape: Rectangle [id:dp9940386237412426] 
\draw  [fill={rgb, 255:red, 255; green, 244; blue, 199}  ,fill opacity=0.52 ] (501.53,107.51) -- (517.16,107.51) -- (517.16,144.55) -- (501.53,144.55) -- cycle ;
%Shape: Rectangle [id:dp9710530307959067] 
\draw  [fill={rgb, 255:red, 255; green, 244; blue, 199}  ,fill opacity=0.52 ] (505.03,110.43) -- (520.66,110.43) -- (520.66,147.47) -- (505.03,147.47) -- cycle ;
%Shape: Rectangle [id:dp6933054471765907] 
\draw  [fill={rgb, 255:red, 255; green, 244; blue, 199}  ,fill opacity=0.52 ] (508.22,113.62) -- (523.85,113.62) -- (523.85,150.66) -- (508.22,150.66) -- cycle ;
%Shape: Rectangle [id:dp13123336612470682] 
\draw  [fill={rgb, 255:red, 255; green, 244; blue, 199}  ,fill opacity=0.52 ] (511.41,116.81) -- (527.04,116.81) -- (527.04,153.84) -- (511.41,153.84) -- cycle ;

%Image [id:dp9325069812683063] 
\draw (52.52,60.31) node  {\includegraphics[width=12.38pt,height=19.04pt]{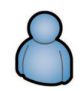}};
%Image [id:dp0711152075860565] 
\draw (53.31,214.11) node  {\includegraphics[width=12.38pt,height=19.04pt]{figures/user.png}};
%Straight Lines from P0 [id:da552622051830157] 
\draw    (63.34,64) -- (115,64) ;
\draw [shift={(117,64)}, rotate = 180] [fill={rgb, 255:red, 0; green, 0; blue, 0 }  ][line width=0.08]  [draw opacity=0] (8.93,-4.29) -- (0,0) -- (8.93,4.29) -- cycle    ;
%Straight Lines from P1 [id:da38350440685398146] 
\draw    (63.34,218) -- (115,218) ;
\draw [shift={(117,218.1)}, rotate = 180] [fill={rgb, 255:red, 0; green, 0; blue, 0 }  ][line width=0.08]  [draw opacity=0] (8.93,-4.29) -- (0,0) -- (8.93,4.29) -- cycle    ;
%Image [id:dp10362239272700946] 
\draw (173,68.11) node  {\includegraphics[width=85pt,height=47pt]{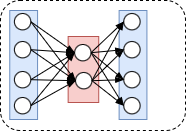}};
%Rounded Rect [id:dp8826131177264377] 
\draw  [fill={rgb, 255:red, 203; green, 221; blue, 243 }  ,fill opacity=1 ][dash pattern={on 0.84pt off 2.51pt}] (65.91,105.73) .. controls (65.91,103.23) and (67.94,101.2) .. (70.44,101.2) -- (168.48,101.2) .. controls (170.98,101.2) and (173.01,103.23) .. (173.01,105.73) -- (173.01,113.82) .. controls (173.01,116.32) and (170.98,118.35) .. (168.48,118.35) -- (70.44,118.35) .. controls (67.94,118.35) and (65.91,116.32) .. (65.91,113.82) -- cycle ;
%Shape: Circle 1 L [id:dp41042610328219586] 
%\draw   (69.24,110.11) .. controls (69.24,106.49) and (72.18,103.56) .. (75.79,103.56) .. controls (79.41,103.56) and (82.34,106.49) .. (82.34,110.11) .. controls (82.34,113.72) and (79.41,116.66) .. (75.79,116.66) .. controls (72.18,116.66) and (69.24,113.72) .. (69.24,110.11) -- cycle ;
\draw (81,110.11) circle [radius=6.72];
%Rounded Rect [id:dp5325330773029058] 
\draw  [fill={rgb, 255:red, 203; green, 221; blue, 243 }  ,fill opacity=1 ][dash pattern={on 0.84pt off 2.51pt}] (360.08,86.65) .. controls (360.08,83.81) and (362.38,81.5) .. (365.23,81.5) -- (460.85,81.5) .. controls (463.69,81.5) and (466,83.81) .. (466,86.65) -- (466,95.85) .. controls (466,98.69) and (463.69,101) .. (460.85,101) -- (365.23,101) .. controls (362.38,101) and (360.08,98.69) .. (360.08,95.85) -- cycle ;
%Shape: Circle 2 [id:dp8197682034725653] 
\draw (375.5,90.83) circle [radius=6.72];
%Shape: Right Angle [id:dp8590140445641944] 
\draw   (328,68) -- (339.01,68) -- (339.01,91.33) ;
%Straight Lines [id:da8709977676683638] 
\draw    (339.01,91.33) -- (358.41,90.8) ;
\draw [shift={(361.41,90.72)}, rotate = 178.44] [fill={rgb, 255:red, 0; green, 0; blue, 0 }  ][line width=0.08]  [draw opacity=0] (8.93,-4.29) -- (0,0) -- (8.93,4.29) -- cycle    ;
%Straight Lines [id:da20293667047683028] 
\draw  [dash pattern={on 4.5pt off 4.5pt}]  (461,171.5) -- (272,170.02) ;
\draw [shift={(269,170)}, rotate = 0.45] [fill={rgb, 255:red, 0; green, 0; blue, 0 }  ][line width=0.08]  [draw opacity=0] (8.93,-4.29) -- (0,0) -- (8.93,4.29) -- cycle    ;
%Straight Lines [id:da9018424372917471] 
\draw  [dash pattern={on 4.5pt off 4.5pt}]  (461,171.5) -- (460.84,136.16) -- (460.78,122.36) ;
%Image [id:dp05574846798318056] 
\draw (173,216.4) node  {\includegraphics[width=85pt,height=47pt]{figures/server.png}};
%Rounded Rect [id:dp19337829238144744] 
\draw  [fill={rgb, 255:red, 203; green, 221; blue, 243 }  ,fill opacity=1 ][dash pattern={on 0.84pt off 2.51pt}] (41.23,133.97) .. controls (41.23,130.95) and (43.68,128.5) .. (46.7,128.5) -- (119.87,128.5) .. controls (122.9,128.5) and (125.34,130.95) .. (125.34,133.97) -- (125.34,143.73) .. controls (125.34,146.75) and (122.9,149.2) .. (119.87,149.2) -- (46.7,149.2) .. controls (43.68,149.2) and (41.23,146.75) .. (41.23,143.73) -- cycle ;
%Rounded Rect [id:dp3571616708624057] 
\draw  [fill={rgb, 255:red, 255; green, 198; blue, 198 }  ,fill opacity=1 ][dash pattern={on 0.84pt off 2.51pt}] (41.59,159.68) .. controls (41.59,156.75) and (43.96,154.38) .. (46.89,154.38) -- (119.38,154.38) .. controls (122.31,154.38) and (124.68,156.75) .. (124.68,159.68) -- (124.68,169.13) .. controls (124.68,172.05) and (122.31,174.42) .. (119.38,174.42) -- (46.89,174.42) .. controls (43.96,174.42) and (41.59,172.05) .. (41.59,169.13) -- cycle ;
%Straight Lines [id:da753402219806243] 
\draw    (346.92,108.2) -- (431,108.49) ;
\draw [shift={(434,108.5)}, rotate = 180.2] [fill={rgb, 255:red, 0; green, 0; blue, 0 }  ][line width=0.08]  [draw opacity=0] (8.93,-4.29) -- (0,0) -- (8.93,4.29) -- cycle    ;
%Straight Lines from image bot [id:da9154326245395499] 
\draw    (228.5,216.2) -- (255,216.5) ;
%Straight Lines from image top [id:da24098434505712096] 
\draw    (228.5,71.2) -- (255,71.2) ;
%Rounded Rect [id:dp8180071929014602] 
\draw  [color={rgb, 255:red, 218; green, 240; blue, 195 }  ,draw opacity=1 ][fill={rgb, 255:red, 218; green, 240; blue, 195 }  ,fill opacity=1 ] (215.33,132.39) .. controls (215.33,130.17) and (217.14,128.37) .. (219.36,128.37) -- (243.97,128.37) .. controls (246.2,128.37) and (248,130.17) .. (248,132.39) -- (248,144.48) .. controls (248,146.7) and (246.2,148.51) .. (243.97,148.51) -- (219.36,148.51) .. controls (217.14,148.51) and (215.33,146.7) .. (215.33,144.48) -- cycle ;
%Straight Lines [id:da12429143293566458] 
\draw [line width=0.75]  [dash pattern={on 4.5pt off 4.5pt}]  (225.42,148.74) -- (225.56,176.76) ;
%Straight Lines [id:da6152386582328147] 
\draw  [dash pattern={on 4.5pt off 4.5pt}]  (225.22,100.97) -- (225.23,112.67) -- (225.55,127.4) ;
\draw [shift={(225.21,97.97)}, rotate = 89.91] [fill={rgb, 255:red, 0; green, 0; blue, 0 }  ][line width=0.08]  [draw opacity=0] (3.57,-1.72) -- (0,0) -- (3.57,1.72) -- cycle    ;
%Straight Lines [id:da6273000288445565] 
\draw  [dash pattern={on 4.5pt off 4.5pt}]  (239.67,172.73) -- (240.04,148.29) ;
\draw [shift={(239.63,175.72)}, rotate = 270.86] [fill={rgb, 255:red, 0; green, 0; blue, 0 }  ][line width=0.08]  [draw opacity=0] (3.57,-1.72) -- (0,0) -- (3.57,1.72) -- cycle    ;
%Straight Lines [id:da3209079146494904] 
\draw [line width=0.75]  [dash pattern={on 3.75pt off 3pt on 7.5pt off 1.5pt}]  (240.04,98.49) -- (240.04,127.51) ;
%Rounded Rect [id:dp66353766759811] 
\draw  [fill={rgb, 255:red, 203; green, 221; blue, 243 }  ,fill opacity=1 ][dash pattern={on 0.84pt off 2.51pt}] (328.91,147.13) .. controls (328.91,144.2) and (331.28,141.83) .. (334.2,141.83) -- (433.38,141.83) .. controls (436.31,141.83) and (438.68,144.2) .. (438.68,147.13) -- (438.68,156.57) .. controls (438.68,159.5) and (436.31,161.87) .. (433.38,161.87) -- (334.2,161.87) .. controls (331.28,161.87) and (328.91,159.5) .. (328.91,156.57) -- cycle ;
%Shape: Circle 3 [id:dp035537743560705715] 
\draw (341.5,152) circle [radius=6.72];
%\draw   (330.24,150.44) .. controls (330.24,146.82) and (333.18,143.89) .. (336.79,143.89) .. controls (340.41,143.89) and (343.34,146.82) .. (343.34,150.44) .. controls (343.34,154.06) and (340.41,156.99) .. (336.79,156.99) .. controls (333.18,156.99) and (330.24,154.06) .. (330.24,150.44) -- cycle ;
%Straight Lines [id:da23198651445991414] 
\draw    (255,143.85) -- (286.96,143.85) ;
\draw [shift={(289.96,143.85)}, rotate = 180] [fill={rgb, 255:red, 0; green, 0; blue, 0 }  ][line width=0.08]  [draw opacity=0] (8.93,-4.29) -- (0,0) -- (8.93,4.29) -- cycle    ;
%Straight Lines [id:da18742783077431868] 
\draw    (255,71.2) -- (255,216.5) ;
%Rounded Rect [id:dp7307251043803731] 
\draw   (290,81.5) .. controls (290,56.65) and (310.15,36.5) .. (335,36.5) -- (570,36.5) .. controls (594.85,36.5) and (615,56.65) .. (615,81.5) -- (615,216.5) .. controls (615,241.35) and (594.85,261.5) .. (570,261.5) -- (335,261.5) .. controls (310.15,261.5) and (290,241.35) .. (290,216.5) -- cycle ;
%Shape: Rectangle [id:dp5316227215622932] 
\draw  [fill={rgb, 255:red, 200; green, 218; blue, 164  }  ,fill opacity=0.52 ] (529.26,109.3) -- (540.53,109.3) -- (540.53,135.5) -- (529.26,135.5) -- cycle ;
%Shape: Rectangle [id:dp879106263903557] 
\draw  [fill={rgb, 255:red, 200; green, 218; blue, 164  }  ,fill opacity=0.52 ] (531.84,111.97) -- (543.11,111.97) -- (543.11,138.17) -- (531.84,138.17) -- cycle ;
%Shape: Rectangle [id:dp9685250959221325] 
\draw  [fill={rgb, 255:red, 200; green, 218; blue, 164  }  ,fill opacity=0.52 ] (534.64,115.16) -- (545.91,115.16) -- (545.91,141.36) -- (534.64,141.36) -- cycle ;
%Shape: Rectangle [id:dp665228502777986] 
\draw  [fill={rgb, 255:red, 200; green, 218; blue, 164  }  ,fill opacity=0.52 ] (537.22,117.83) -- (548.49,117.83) -- (548.49,144.03) -- (537.22,144.03) -- cycle ;
%Shape: Rectangle [id:dp6485757490303001] 
\draw  [fill={rgb, 255:red, 200; green, 218; blue, 164  }  ,fill opacity=0.52 ] (539.8,120.5) -- (551.07,120.5) -- (551.07,146.7) -- (539.8,146.7) -- cycle ;

%Rounded Rect Conv Max R [id:dp20749751220441892] 
\draw  [dash pattern={on 0.84pt off 2.51pt}] (490.09,110.24) .. controls (490.09,102.7) and (496.2,96.59) .. (503.74,96.59) -- (596.08,96.59) .. controls (603.62,96.59) and (609.73,102.7) .. (609.73,110.24) -- (609.73,151.18) .. controls (609.73,158.71) and (603.62,164.82) .. (596.08,164.82) -- (503.74,164.82) .. controls (496.2,164.82) and (490.09,158.71) .. (490.09,151.18) -- cycle ;
%Shape: Rectangle [id:dp20740598673987498] 
\draw  [fill={rgb, 255:red, 255; green, 244; blue, 199}  ,fill opacity=0.52 ] (552.98,105.87) -- (568.61,105.87) -- (568.61,142.91) -- (552.98,142.91) -- cycle ;
%Shape: Rectangle [id:dp26285966809689165] 
\draw  [fill={rgb, 255:red, 255; green, 244; blue, 199}  ,fill opacity=0.52 ] (556.16,109.05) -- (571.8,109.05) -- (571.8,146.09) -- (556.16,146.09) -- cycle ;
%Shape: Rectangle [id:dp5000960242541915] 
\draw  [fill={rgb, 255:red, 255; green, 244; blue, 199}  ,fill opacity=0.52 ] (559.67,111.98) -- (575.3,111.98) -- (575.3,149.02) -- (559.67,149.02) -- cycle ;
%Shape: Rectangle [id:dp5157349333596732] 
\draw  [fill={rgb, 255:red, 255; green, 244; blue, 199}  ,fill opacity=0.52 ] (562.85,115.17) -- (578.49,115.17) -- (578.49,152.2) -- (562.85,152.2) -- cycle ;
%Shape: Rectangle [id:dp5306394684531319] 
\draw  [fill={rgb, 255:red, 255; green, 244; blue, 199}  ,fill opacity=0.52 ] (566.04,118.35) -- (581.67,118.35) -- (581.67,155.39) -- (566.04,155.39) -- cycle ;

%Shape: Rectangle [id:dp7752750425567204] 
\draw  [fill={rgb, 255:red, 200; green, 218; blue, 164  }  ,fill opacity=0.52 ] (583.89,110.84) -- (595.17,110.84) -- (595.17,137.05) -- (583.89,137.05) -- cycle ;
%Shape: Rectangle [id:dp8044891878327234] 
\draw  [fill={rgb, 255:red, 200; green, 218; blue, 164  }  ,fill opacity=0.52 ] (586.47,113.51) -- (597.75,113.51) -- (597.75,139.71) -- (586.47,139.71) -- cycle ;
%Shape: Rectangle [id:dp5762894877361086] 
\draw  [fill={rgb, 255:red, 200; green, 218; blue, 164  }  ,fill opacity=0.52 ] (589.28,116.7) -- (600.55,116.7) -- (600.55,142.91) -- (589.28,142.91) -- cycle ;
%Shape: Rectangle [id:dp2770059352584703] 
\draw  [fill={rgb, 255:red, 200; green, 218; blue, 164  }  ,fill opacity=0.52 ] (591.85,119.37) -- (603.13,119.37) -- (603.13,145.58) -- (591.85,145.58) -- cycle ;
%Shape: Rectangle [id:dp3509226364787814] 
\draw  [fill={rgb, 255:red, 200; green, 218; blue, 164  }  ,fill opacity=0.52 ] (594.43,122.04) -- (605.71,122.04) -- (605.71,148.24) -- (594.43,148.24) -- cycle ;

%Rounded Rect [id:dp7860502677516374] 
\draw  [fill={rgb, 255:red, 255; green, 244; blue, 199}  ,fill opacity=0.63 ][dash pattern={on 0.84pt off 2.51pt}] (512.32,180.64) .. controls (512.32,177.62) and (514.77,175.17) .. (517.79,175.17) -- (590.96,175.17) .. controls (593.98,175.17) and (596.43,177.62) .. (596.43,180.64) -- (596.43,190.4) .. controls (596.43,193.42) and (593.98,195.87) .. (590.96,195.87) -- (517.79,195.87) .. controls (514.77,195.87) and (512.32,193.42) .. (512.32,190.4) -- cycle ;
%Rounded Rect [id:dp00832262147957108] 
\draw  [fill={rgb, 255:red, 200; green, 218; blue, 164  }  ,fill opacity=0.35 ][dash pattern={on 0.84pt off 2.51pt}] (512.68,206.34) .. controls (512.68,203.42) and (515.05,201.05) .. (517.97,201.05) -- (590.47,201.05) .. controls (593.39,201.05) and (595.77,203.42) .. (595.77,206.34) -- (595.77,215.79) .. controls (595.77,218.72) and (593.39,221.09) .. (590.47,221.09) -- (517.97,221.09) .. controls (515.05,221.09) and (512.68,218.72) .. (512.68,215.79) -- cycle ;
%Straight Lines [id:da27634482180272846] 
\draw  [dash pattern={on 4.5pt off 4.5pt}]  (269.8,239.2) -- (271.04,82.16) -- (270.6,51.6) ;
%Straight Lines [id:da17476059655056653] 
\draw    (615.67,129.71) -- (646.67,129.71) ;
\draw [shift={(649.67,129.71)}, rotate = 180] [fill={rgb, 255:red, 0; green, 0; blue, 0 }  ][line width=0.08]  [draw opacity=0] (8.93,-4.29) -- (0,0) -- (8.93,4.29) -- cycle    ;
%Straight Lines [id:da8257233225867722] 
\draw    (650.5,19.75) -- (650,279.95) ;
%Straight Lines [id:da29608816516918834] 
\draw    (650,279.95) -- (49.89,279.42) ;
%Straight Lines [id:da39867238229688395] 
\draw    (49.89,279.42) -- (50.11,249.78) ;
%Straight Lines [id:da43833755867728696] 
\draw    (650,20.08) -- (51,20.63) ;
%Straight Lines [id:da14372106070598256] 
\draw    (51,20.63) -- (51.22,41.67) ;
%Straight Lines [id:da9325803382600018] 
\draw  [dash pattern={on 4.5pt off 4.5pt}]  (637.4,30.6) -- (63.2,31.09) ;
%Straight Lines [id:da45133007817016124] 
\draw  [dash pattern={on 4.5pt off 4.5pt}]  (63.2,32.09) -- (63.41,51.27) ;
%Straight Lines [id:da9152500773693897] 
\draw  [dash pattern={on 4.5pt off 4.5pt}]  (637.4,30.6) -- (636.2,270.4) ;
%Straight Lines [id:da35498995513416187] 
\draw  [dash pattern={on 4.5pt off 4.5pt}]  (636.2,270.4) -- (70.2,270.18) ;
%Straight Lines [id:da9155127928342783] 
\draw  [dash pattern={on 4.5pt off 4.5pt}]  (70.2,270.18) -- (70.21,240.47) ;
%Straight Lines [id:da14873806304531445] 
\draw  [dash pattern={on 4.5pt off 4.5pt}]  (619.22,150.38) -- (636.8,150.5) ;
\draw [shift={(616.22,150.36)}, rotate = 0.4] [fill={rgb, 255:red, 0; green, 0; blue, 0 }  ][line width=0.08]  [draw opacity=0] (8.93,-4.29) -- (0,0) -- (8.93,4.29) -- cycle    ;
%Straight Lines [id:da2713233698779641] 
\draw  [dash pattern={on 4.5pt off 4.5pt}]  (63.41,51.27) -- (121.13,51.26) ;
%Straight Lines [id:da5180836914436369] 
\draw  [dash pattern={on 4.5pt off 4.5pt}]  (70.21,240.47) -- (111,240.8) ;
%Dashed Straight Lines to image bot [id:da00183405564553496] 
\draw  [dash pattern={on 4.2pt off 3pt}]  (233,241.2) -- (269.8,241.2) ;
\draw [shift={(229,240.87)}, rotate = 0.47] [fill={rgb, 255:red, 0; green, 0; blue, 0 }  ][line width=0.08]  [draw opacity=0] (8.93,-4.29) -- (0,0) -- (8.93,4.29) -- cycle    ;
%Dashed Straight Lines to image top [id:da41269774015401883] 
\draw  [dash pattern={on 5pt off 4pt}]  (233,51.6) -- (268.2,51.6) ;
\draw [shift={(229,51.27)}, rotate = 0.47] [fill={rgb, 255:red, 0; green, 0; blue, 0 }  ][line width=0.08]  [draw opacity=0] (8.93,-4.29) -- (0,0) -- (8.93,4.29) -- cycle    ;
%Rounded Rect [id:dp5102455538795369] 
\draw  [fill={rgb, 255:red, 203; green, 221; blue, 243 }  ,fill opacity=1 ][dash pattern={on 0.84pt off 2.51pt}] (492.91,72.4) .. controls (492.91,69.9) and (494.94,67.87) .. (497.44,67.87) -- (595.48,67.87) .. controls (597.98,67.87) and (600.01,69.9) .. (600.01,72.4) -- (600.01,80.48) .. controls (600.01,82.99) and (597.98,85.01) .. (595.48,85.01) -- (497.44,85.01) .. controls (494.94,85.01) and (492.91,82.99) .. (492.91,80.48) -- cycle ;
%Straight Lines [id:da6086329455909946] 
\draw    (496.74,91.59) -- (544,91.75) -- (609,91.99) ;
\draw [shift={(612,92)}, rotate = 180.21] [fill={rgb, 255:red, 0; green, 0; blue, 0 }  ][line width=0.08]  [draw opacity=0] (8.93,-4.29) -- (0,0) -- (8.93,4.29) -- cycle    ;
%Shape: Circle 1 R [id:dp10776502922343345] 
\draw (508.5,76.5) circle [radius=6.72];
%\draw   (497.33,75.5) .. controls (497.33,71.91) and (500.24,69) .. (503.83,69) .. controls (507.42,69) and (510.33,71.91) .. (510.33,75.5) .. controls (510.33,79.09) and (507.42,82) .. (503.83,82) .. controls (500.24,82) and (497.33,79.09) .. (497.33,75.5) -- cycle ;
%Straight Lines [id:da24858669149356205] 
\draw    (474.16,61.31) .. controls (475.82,62.98) and (475.81,64.65) .. (474.14,66.31) .. controls (472.47,67.97) and (472.46,69.64) .. (474.12,71.31) .. controls (475.78,72.98) and (475.77,74.65) .. (474.1,76.31) .. controls (472.43,77.97) and (472.42,79.64) .. (474.08,81.31) .. controls (475.73,82.98) and (475.72,84.65) .. (474.05,86.31) .. controls (472.38,87.97) and (472.37,89.64) .. (474.03,91.31) .. controls (475.69,92.98) and (475.68,94.65) .. (474.01,96.31) .. controls (472.34,97.97) and (472.33,99.64) .. (473.99,101.31) .. controls (475.65,102.98) and (475.64,104.65) .. (473.97,106.31) .. controls (472.3,107.97) and (472.29,109.64) .. (473.95,111.31) .. controls (475.61,112.98) and (475.6,114.65) .. (473.93,116.31) .. controls (472.26,117.97) and (472.25,119.64) .. (473.91,121.31) .. controls (475.56,122.98) and (475.55,124.65) .. (473.88,126.31) .. controls (472.21,127.97) and (472.2,129.64) .. (473.86,131.31) .. controls (475.52,132.98) and (475.51,134.65) .. (473.84,136.31) .. controls (472.17,137.97) and (472.16,139.64) .. (473.82,141.31) .. controls (475.48,142.98) and (475.47,144.65) .. (473.8,146.31) .. controls (472.13,147.97) and (472.12,149.64) .. (473.78,151.31) .. controls (475.44,152.98) and (475.43,154.65) .. (473.76,156.31) .. controls (472.09,157.97) and (472.08,159.64) .. (473.74,161.31) .. controls (475.4,162.98) and (475.39,164.65) .. (473.72,166.31) .. controls (472.05,167.97) and (472.04,169.64) .. (473.69,171.31) .. controls (475.35,172.98) and (475.34,174.65) .. (473.67,176.31) .. controls (472,177.97) and (471.99,179.64) .. (473.65,181.31) .. controls (475.31,182.98) and (475.3,184.65) .. (473.63,186.31) .. controls (471.96,187.97) and (471.95,189.64) .. (473.61,191.31) .. controls (475.27,192.98) and (475.26,194.65) .. (473.59,196.31) .. controls (471.92,197.97) and (471.91,199.64) .. (473.57,201.31) .. controls (475.23,202.98) and (475.22,204.65) .. (473.55,206.31) .. controls (471.88,207.97) and (471.87,209.64) .. (473.52,211.31) .. controls (475.18,212.98) and (475.17,214.65) .. (473.5,216.31) .. controls (471.83,217.97) and (471.82,219.64) .. (473.48,221.31) .. controls (475.14,222.98) and (475.13,224.65) .. (473.46,226.31) .. controls (471.79,227.97) and (471.78,229.64) .. (473.44,231.31) .. controls (475.1,232.98) and (475.09,234.65) .. (473.42,236.31) -- (473.4,240.8) -- (473.4,240.8) ;
%Straight Lines [id:da7700222573857974] 
\draw    (478.16,61.11) .. controls (479.82,62.78) and (479.81,64.45) .. (478.14,66.11) .. controls (476.47,67.77) and (476.46,69.44) .. (478.12,71.11) .. controls (479.78,72.78) and (479.77,74.45) .. (478.1,76.11) .. controls (476.43,77.77) and (476.42,79.44) .. (478.08,81.11) .. controls (479.73,82.78) and (479.72,84.45) .. (478.05,86.11) .. controls (476.38,87.77) and (476.37,89.44) .. (478.03,91.11) .. controls (479.69,92.78) and (479.68,94.45) .. (478.01,96.11) .. controls (476.34,97.77) and (476.33,99.44) .. (477.99,101.11) .. controls (479.65,102.78) and (479.64,104.45) .. (477.97,106.11) .. controls (476.3,107.77) and (476.29,109.44) .. (477.95,111.11) .. controls (479.61,112.78) and (479.6,114.45) .. (477.93,116.11) .. controls (476.26,117.77) and (476.25,119.44) .. (477.91,121.11) .. controls (479.56,122.78) and (479.55,124.45) .. (477.88,126.11) .. controls (476.21,127.77) and (476.2,129.44) .. (477.86,131.11) .. controls (479.52,132.78) and (479.51,134.45) .. (477.84,136.11) .. controls (476.17,137.77) and (476.16,139.44) .. (477.82,141.11) .. controls (479.48,142.78) and (479.47,144.45) .. (477.8,146.11) .. controls (476.13,147.77) and (476.12,149.44) .. (477.78,151.11) .. controls (479.44,152.78) and (479.43,154.45) .. (477.76,156.11) .. controls (476.09,157.77) and (476.08,159.44) .. (477.74,161.11) .. controls (479.4,162.78) and (479.39,164.45) .. (477.72,166.11) .. controls (476.05,167.77) and (476.04,169.44) .. (477.69,171.11) .. controls (479.35,172.78) and (479.34,174.45) .. (477.67,176.11) .. controls (476,177.77) and (475.99,179.44) .. (477.65,181.11) .. controls (479.31,182.78) and (479.3,184.45) .. (477.63,186.11) .. controls (475.96,187.77) and (475.95,189.44) .. (477.61,191.11) .. controls (479.27,192.78) and (479.26,194.45) .. (477.59,196.11) .. controls (475.92,197.77) and (475.91,199.44) .. (477.57,201.11) .. controls (479.23,202.78) and (479.22,204.45) .. (477.55,206.11) .. controls (475.88,207.77) and (475.87,209.44) .. (477.52,211.11) .. controls (479.18,212.78) and (479.17,214.45) .. (477.5,216.11) .. controls (475.83,217.77) and (475.82,219.44) .. (477.48,221.11) .. controls (479.14,222.78) and (479.13,224.45) .. (477.46,226.11) .. controls (475.79,227.77) and (475.78,229.44) .. (477.44,231.11) .. controls (479.1,232.78) and (479.09,234.45) .. (477.42,236.11) -- (477.4,240.6) -- (477.4,240.6) ;

\draw (65,70) node [anchor=north west][inner sep=0.75pt]  []  {${\textstyle \langle x_{pub} \rangle _{0}}$};

\draw (44.27,73) node [anchor=north west][inner sep=0.75pt]    {$P_{0}$};

\draw (45.06,226.8) node [anchor=north west][inner sep=0.75pt]    {$P_{1}$};

\draw (65,200) node [anchor=north west][inner sep=0.75pt]  []  {${\textstyle \langle x_{pub} \rangle _{1}}$};

\draw (203,80) node [anchor=north west][inner sep=0.75pt]    {${\displaystyle \langle w \rangle _{0}}$};

\draw (203,190) node [anchor=north west][inner sep=0.75pt]    {${\displaystyle \langle w \rangle _{1}}$};

\draw (301.54,55.27) node [anchor=north west][inner sep=0.75pt]    {${\displaystyle \langle y\rangle }$};

\draw (301.54,80.33) node [anchor=north west][inner sep=0.75pt]    {${\textstyle \langle \hat{y} \rangle }$};

\draw (443.38,104.33) node [anchor=north west][inner sep=0.75pt]    {$\mathbf{J}$};

\draw (295.21,171.17) node [anchor=north west][inner sep=0.75pt]    {${\displaystyle \langle \hat{y} \rangle \ =\ s( a)}$};

\draw (309.88,110.33) node [anchor=north west][inner sep=0.75pt]    {${\textstyle w^{n+1}\leftarrow w-\eta \nabla _{w}\mathcal{L})}$};

\draw (300.71,219.33) node [anchor=north west][inner sep=0.75pt]    {${\textstyle \nabla _{w}\mathcal{L} \ =\ \left[\frac{\partial \mathcal{L}}{\partial w}\right]}$};

\draw (299.21,197) node [anchor=north west][inner sep=0.75pt]    {$J\ \leftarrow \mathcal{L}( \langle \hat{y} \rangle \ -\ \langle y\rangle )$};

\draw (130,132) node [anchor=north west][inner sep=0.75pt]    {$z=w*x+b$};

\draw (130,157) node [anchor=north west][inner sep=0.75pt]    {${\displaystyle a=ReLU( z)}$};

\draw (44.58,131.52) node [anchor=north west][inner sep=0.75pt]   [align=left] {{ Linear Layer}};

\draw (46.76,158.01) node [anchor=north west][inner sep=0.75pt]   [align=left] {{ ReLU Layer}};

\draw (90.19,103.69) node [anchor=north west][inner sep=0.75pt]   [align=left] {{ Forward Prop.}};

\draw (72.5,104) node [anchor=north west][inner sep=0.75pt]   [align=left] {{ 1}};

\draw (383,85.67) node [anchor=north west][inner sep=0.75pt]   [align=left] {{  Compute Cost}};

\draw (367,86) node [anchor=north west][inner sep=0.75pt]   [align=left] {{ 2}};

\draw (216.59,132.01) node [anchor=north west][inner sep=0.75pt]   [align=left] {\textit{{ 2PC}}};

\draw (345.17,146) node [anchor=north west][inner sep=0.75pt]   [align=left] {{ Backward Prop.}};

\draw (333,146.5) node [anchor=north west][inner sep=0.75pt]   [align=left] {{ 3}};

\draw (523.67,180.19) node [anchor=north west][inner sep=0.75pt]   [align=left] {{ Conv Layer}};

\draw (526.84,206.68) node [anchor=north west][inner sep=0.75pt]   [align=left] {{ Max Layer}};

\draw (513.19,70) node [anchor=north west][inner sep=0.75pt]   [align=left] {{ Forward Prop.}};

\draw (500,72) node [anchor=north west][inner sep=0.75pt]   [align=left] {{ 1}};
\end{tikzpicture}
 \end{adjustbox}
    \caption{Utilizing function secret sharing between two servers for \nick}
    \label{fig:serverside_FSS}
   
\end{figure*}

\subsection{Training the \nick{} protocol}
\label{subsec: protocol u-shapedsl}

We have used~\autoref{alg:clientusl} and~\autoref{alg:serveruspl} to train \nick. 
It consists of: 

\noindent \textit{Initialization phase.} \enskip
The initialization phase occurs only once, whereas the other phases continue until the model iterates through all epochs. This phase consists of socket %initialization 
and random weight initialization $\Phi$. %weighting. 
The client establishes a socket connection and synchronizes the hyperparameters $\eta, n, N, E$. $ATm$ and the gradient are initially set to zero $\emptyset$. These parameters must be synchronized on both sides to be trained in the same way. %As described in \autoref{subsec: protocol u-shapedsl}, 
In \nick{}, the client executes the first few layers as well as the last layer of the model, %(a distinction from MSnH), 
while the remaining layers are executed on the server-side. 

%\subsubsection{Forward propagation}
\noindent \textit{Forward propagation.} \enskip
In \autoref{alg:clientusl}, the client carries out forward propagation on the input $x$ using $f_{\theta_{C}}$ to calculate $ATm$. To prevent %the%
client data privacy, %privacy client data,
the private input %which is%
($ATm$) is first masked using random mask $\alpha$ to construct the public input $x_{pub}$ before sending it to the server. Once $x_{pub}$ is constructed, the next step is to send $x_{pub}$ to the server. %To prevent the privacy client data, a random mask $\alpha$ is added to $ATm$ to generate a public input $x_{pub}=ATm + \alpha$, preventing the servers from accessing $ATm$. 
Once $x_{pub}$ is computed, the client sends it to the servers to continue the forward propagation. As can be seen in \autoref{fig:serverside_FSS}, we have two servers labeled as $P_{0}$ and $P_{1}$. Each server receives the public input $x_{pub}$ and initiates forward propagation by executing their respective portion of the model. On the server-side, we employ FSS, %allowing us 
to compute the function keys. Specifically, these keys are generated for the ReLU layers. We employ the FSS $\mathsf{KeyGen}$ algorithm, to create these keys for each server. 
Additionally, to compute the FC layers %$ f(x_{pub})$, and $h(x_{pub})$ 
on the server-side, we used beaver triples. % as outlined in \autoref{alg:beaver}.
As depicted in \autoref{alg:serveruspl}, each server executes their share of the layer on $x_{pub}$ until the $L-1^{th}$ layers. After executing the $L-1^{th}$ layer, each server sends the output to the client. Upon reception, the client processes the final output layer %of the model 
to get the predicted output $\hat{y}$. The next step involves calculating the loss, and this is computed using the equation $J \leftarrow \mathcal{L} (\hat{y}_{j}, y_{j})$. In U-shaped SL, the loss calculation involves multiple steps. \textit{First, the model's output is represented in fixed precision. This output represents the actual values without any secret sharing, as it remains on the client-side. Next, the target values are encrypted into additive shares. These encrypted shares are then used to compute the loss function, which also remains in the form of additive shares.} These loss shares are sent back to the servers for further computation.

The model's privacy %privacy of the model%
is preserved by using additive shares for both the target values and the loss. %the privacy of the model is preserved.%
The target values remain encrypted and distributed across the servers, ensuring that no individual server can access the original target values. Similarly, the loss remains encrypted throughout the computation, preventing any potential privacy leaks.

%\subsubsection{Backward Propagation}
\noindent \textit{Backward propagation.} \enskip
After calculating the loss the client starts the backward propagation by computing ${\frac{\partial J}{\partial \hat{\mathbf{y}}_{j}}}$ and ${\frac{\partial J}{\partial ATm_{j}^{(L)} } }$ using the chain rule and sends the values to the server. Upon reception the servers first compute ${\frac{\partial J}{\partial ATm^{L-1}_{j}}}$ and continue the backward propagation until layer $l+1$ (see~\autoref{alg:serveruspl}). Both the servers update the weights ($\boldsymbol{w}$) and biases ($\boldsymbol{b}$) of their respective layers using 
$\boldsymbol{w} =  \boldsymbol{w} - \eta\frac{\partial J}{\partial \boldsymbol{w}}$ and $ \boldsymbol{b} = \boldsymbol{b} - \eta\frac{\partial J}{\partial \boldsymbol{b}}$.
After updating $\boldsymbol{w}$ and $\boldsymbol{b}$ of the layer $l+1$, the server calculates ${\frac{\partial J}{\partial ATm^{l+1}_{j}}}$ and sends it to the client. Upon reception, the client calculates the gradient of $J$ with respect to $\boldsymbol{w}$ and $\boldsymbol{b}$ of the layers residing on the client-side and updates $\boldsymbol{w}$ and $\boldsymbol{b}$. The forward and backward propagation continue until the model converges to learn a suitable set of parameters.

% \begin{align}
% \label{equ:serverUpdateWB}
% 	\boldsymbol{w} =  \boldsymbol{w} - \eta\frac{\partial J}{\partial \boldsymbol{w}}, \quad & \boldsymbol{b} = \boldsymbol{b} - \eta\frac{\partial J}{\partial \boldsymbol{b}} 
% \end{align}

% {\SetAlgoNoLine%
% \begin{minipage}{0.5\textwidth}
\begin{algorithm}[!hbt]
%\footnotesize
\scriptsize
%\KwResult{Write here the result }
 % \textbf{Initialization:}\\
 $soc\leftarrow$ socket initialized with port and address\;
 % $soc.\mathsf{connect}$\\
 $\eta, n, N, E \leftarrow soc.\mathsf{synchronize}$\\
 $ \{\boldsymbol{w}^{( i)}, \boldsymbol{b}^{( i)}\}_{\forall i\in \{0..l\}} \ \leftarrow$  initialize using $\Phi $\\
 $%\{\mathbf{z}^{( i)}\}_{\forall i\in \{0..l\}} ,
 \{ATm^{( i)}\}_{\forall i\in \{0..l\}}\leftarrow \emptyset \ $, 
 $% \left\{\frac{\partial J}{\partial \mathbf{z}^{( i)}}\right\}_{\forall i\in \{0..l\}} ,
 \left\{\frac{\partial J}{\partial ATm^{( i)}}\right\}_{\forall i\in \{0..l\}}\leftarrow \emptyset \ $\\
 \For{$\displaystyle e \ \in \ E $}{
 	\For{$\displaystyle \text{each} \ \text{batch}\ ( x,\ y) \ \text{generated\ from}\ D\ $}{
 	$\displaystyle  \mathbf{Forward\ propagation:}$ \\
          % $\displaystyle \ \ \ \ O.zero\_grad()  $\\
 	% $\displaystyle \ \ \ \ \mathbf{a}^{0} \ \ \leftarrow \mathbf{x}$ \\
 	% \For{$i \leftarrow 1$ to $l$}{$\displaystyle \ \ 
  % $
 	% $\displaystyle \ \ \ \ \ \ \ \ \mathbf{z}^{( i)} \ \leftarrow \ f^{( i)}\left( ATm^{( i-1)}\right)$\\
 	% $\displaystyle \ \ \ \ \ \ \ \ ATm^{( i)} \ \leftarrow \ g^{( i)}\left( \mathbf{z}^{( i)}\right)$\\}
  	\For{$i \leftarrow 1$ to $l$}{$\displaystyle$
% 	$\displaystyle  \mathbf{z}^{( i)} \ \leftarrow \ c^{( i)}\left( \mathbf{a}^{( i-1)}\right)$\\
 	$\displaystyle ATm^{( i)} \ \leftarrow \ f_{\theta_{C}}\left( x_{pub}^{( i)}\right)$\\}
$\displaystyle  x_{pub}^{i} \ \leftarrow \ ATm^{( i)} + \alpha$\\
 	$\displaystyle   soc.\mathsf{send}\ ( x_{pub}^{i})$\\
 	$\displaystyle   soc.\mathsf{receive}\ ( ATm^{L-1}_{j})$\\
  $\displaystyle   \text{Compute}\ ( ATm^{L-1})$\\
  $\displaystyle   \hat{y} \ \leftarrow \ s^{L}\left(ATm^{( L-1)}\right)$\\

 	$\displaystyle  J \leftarrow \mathcal{L} (\hat{\mathbf{y}}, \mathbf{y_{j}})$\\

 	$\displaystyle \mathbf{Backward\ propagation:}$\\
	$\displaystyle \text{Compute}\left\{\frac{\partial J}{\partial \hat{\mathbf{y}}_{j}} \& \frac{\partial J}{\partial ATm_{j}^{(L)} } \right\}$\\
	$\displaystyle  soc.\mathsf{send}\ \left( \frac{\partial J}{\partial ATm_{j}^{(L)}}, \frac{\partial J}{\partial \mathbf{w}_{j}^{(L)}} \right)$\\
 
 	$\displaystyle  soc.\mathsf{receive}\ \left( \frac{\partial J}{\partial ATm_{j}^{(l+1)}} \right)$\\
   	$\displaystyle \text{Compute}\left\{\frac{\partial J}{\partial ATm^{l}}\right\}$\\
 	\For{$i\leftarrow l$ to $1$}{$\text{Compute}\ \left\{ \frac{\partial J}{\partial \boldsymbol{w}^{( i)}}, \ \frac{\partial J}{\partial \boldsymbol{b}^{( i)}} \right\}$ and
 	$\displaystyle\  \text{Update}\ \boldsymbol{w}^{( i)},\ \boldsymbol{b}^{( i)}$
 	}
 	}
 }
 \caption{\textbf{Client-Side}}
 \label{alg:clientusl}
\end{algorithm}	%}
% \end{minipage}
% \begin{minipage}{0.5\textwidth}
%{\SetAlgoNoLine%
\begin{algorithm}[!hbt]
\scriptsize
%\KwResult{Write here the result }
 % \textbf{Initialization:}\\
 $soc\leftarrow$ socket initialized with port and address\;
 % $soc.\mathsf{connect}$\\
 $\eta, n, N, E \leftarrow soc.\mathsf{synchronize}$\\

 \For{$\displaystyle e \ \in \ E $}{
 	\For{$\displaystyle i \leftarrow l+1 \ \mathbf{to} \ L - 1 \ $}{
 	$\displaystyle \mathbf{Forward\ propagation:}$\\
  
 	$\displaystyle soc.\mathsf{receive}\ (x_{pub}^{i}) \ \ $ \\
  \For{$\displaystyle j= 0, 1  $}{
$\displaystyle  ATm_{j}^{L-1} \ \leftarrow \ f_{\theta_{Pj}}(x_{pub})$}
$\displaystyle soc.\mathsf{send} \ \left( ATm^{L-1}_{j}\right)$\\
  
$\displaystyle \mathbf{Backward\ propagation:}$\\
	$\displaystyle  soc.\mathsf{receive}\ \left( \frac{\partial J}{\partial ATm_{j}^{(L)}}, \frac{\partial J}{\partial \mathbf{w}_{j}^{(L)}} \right)$\\
 \For{$\displaystyle j= 0, 1  $}{	
  $\displaystyle \text{Compute}\left\{\frac{\partial J}{\partial ATm^{L-1}_{j}}\right\}$\\
\For{$i\leftarrow L-1$ to $l+1$}{$\text{Compute}\ \left\{ \frac{\partial J}{\partial \boldsymbol{w}_{j}^{( i)}}, \ \frac{\partial J}{\partial \boldsymbol{b}_{j}^{( i)}} \right\}$ and
 	$\displaystyle\ \text{Update}\ \boldsymbol{w}_{j}^{( i)},\ \boldsymbol{b}^{( i)}$}
$\displaystyle soc.\mathsf{send} \ \left( \frac{\partial J}{\partial ATm_{j}^{(l+1)}}\right)$\\
 		}}
 }
 \caption{\textbf{Server-Side}}
 \label{alg:serveruspl}
\normalsize
\end{algorithm}%}

\section{Threat Model \& Security Analysis}
\label{sec:threat_model_and_sec}

% This section defines threat model and proves the protocol's security.% of our protocol. 

% \subsection{Threat model}
% \label{subsec:threat_model}

\textbf{Threat model:} In SL, collaborative learning of local models introduces potential risk due to interactions among the participants~\cite{li2022ressfl}.
In this study, we examine a scenario involving a semi-honest adversary, denoted as $\mathcal{ADV}$, capable of corrupting one of the two servers engaged in our protocol. In this context, the corrupted parties adhere to the protocol's specifications, all while attempting to gather as much information as possible about the
inputs and function shares of other participants. Within our specific framework, secure protocols are employed to facilitate collaborative CNN model training by three parties: a client and two servers. The client, who also serves as the data owner (note that considering a malicious client in this context is pointless), initiates the process by executing the initial layers of the model. This action produces an intermediary result, denoted as $ATm$, which is subsequently kept confidential through FSS %secret sharing 
between two independent servers. Following this, these two servers proceed to jointly train the remaining segments of the model using the client's data, achieved by employing the FSS protocol.

As such, behavior of $\mathcal{ADV}$ can be summarized as: 

\begin{itemize}
    \item $\mathcal{ADV}$ does not possess information about the participating
clients, %participating in distributed training, 
except necessary details required to execute the SL protocol. 
\item $\mathcal{ADV}$ is familiar with a \textit{shadow} dataset ($\tilde{x}$), %which captures 
from the same domain as the clients' training set $x$. In this case no element of $\tilde{x}$ overlap with $x$, i.e. $\tilde{x} \cap x = \emptyset$.
\item $\mathcal{ADV}$ has no knowledge of the architecture and weights of the whole model. However, %the $\mathcal{ADV}$ 
she can gather information and make assumptions of the client model and train an architecturally similar model $\tilde{f_{\theta_C}}$ using the $\tilde{x}$.
%\item  $\mathcal{ADV}$ is unaware of the specific task on which the distributed model is trained.
\end{itemize}

% Finally, we extend the above threat model by defining attacks available to $\mathcal{ADV}$. %, namely the LIA and MIA. 
% Additionally, FSHA aims to take control of the feature space of a SL model, however we do not cover the security analysis against these types of attacks as the authors of MSnH~\cite{khan2024make} have proven the security of FSS and SL against it.

\begin{myAttack}[Label Inference Attack] Let $\mathcal{ADV}$ represents an adversary acting as one of the servers. $\mathcal{ADV}$ successfully performs a LIA if she manages to infer the private labels %with perfect accuracy 
by exploiting the client's output layer; if the client computes only the output layer locally and shares it with the server, the private labels can be inferred.
\end{myAttack}

\begin{myAttack}[Model Inversion Attack] Let $\mathcal{ADV}$ represents an adversary acting as one of the servers. $\mathcal{ADV}$ successfully performs MIA 
if she manages to infer input samples using client's model output. $\mathcal{ADV}$ use these inferred input samples to train a new model that is functionally similar to $f_{\theta_{c}}$.
\end{myAttack}

\subsection{Security analysis}
\label{sec: security-analysis}

With our threat model established, we can now move forward to demonstrate the security of our protocol.

\begin{proposition}[LIA Soundness]
\label{proposition:lia-soundness}
   Let $\mathcal{ADV}$ be a semi-honest adversary that corrupts at most one of the two servers ($P_0$ or $P_1$) involved in the protocol. Then $\mathcal{ADV}$ cannot launch a successful Label Inference attack.
\end{proposition}
%\vspace{-4em}
\begin{proof}
Let $\mathcal{ADV}$ be a semi-honest adversary that has compromised $P_0$ or $P_1$. In addition to that, assume $\mathcal{ADV}$ gains access to:
\begin{inparaenum}[i)]
    \item \textit{The public input $x_{pub} = ATm + \alpha$,}
    \item \textit{One part of $f_{\theta_{P_{b}}}$, where $b$ is either~0 or~1,}
    \item \textit{An architecturally similar model $\tilde{f_{\theta_C}}$ to the client's model\footnote{In an optimal case for LIA, the user has a single layer model, however, this is not a realistic assumption in all applications and we assume the clients model can contain an arbitrary number of layers, which can greatly reduce the accuracy of LIA~\cite{luo2023feature}.} $f_{\theta_C}$ that produces gradients $\tilde{\Delta}$,}
    \item \textit{The masked gradients $\Delta_{pub} = \Delta + \alpha$, where $\Delta$ is the back propagation gradients and $\alpha$ is a sufficiently large random mask generated in FSS.}
\end{inparaenum}
Then, if $\mathcal{ADV}$ cannot train an architecturally similar model $\tilde{f_{\theta_C}}$ that can produce similar gradients to %client's model 
$f_{\theta_C}$, %then%
our protocol is secure against LIA.
%Following our assumptions, let us consider that $\mathcal{ADV}$ trains an architecturally similar model $\tilde{f_{\theta_C}}$ that produces gradients $\tilde{\Delta}$. Then $\mathcal{ADV}$ would require to compare $\tilde{\Delta}$ with $\Delta_{pub}$ and choose the label that has the closest gradient value. This means that $\tilde{\Delta}$ should be similar to $\Delta$. However, the opposite is true, because $\Delta_{pub} = \Delta + \alpha $ and $\Delta + \alpha \neq \tilde{\Delta}$, due to the existence of the random mask $\alpha$. Additionally,  $\mathcal{ADV}$ cannot obtain the exact value of $\alpha$ applied to the gradients $\Delta$, due to this she will not be able to accurately compare $\Delta_{pub}$ with $\tilde{\Delta}$. This is because $\Delta_{pub} > \tilde{\Delta}$ leads %ing%
%to inferring the same class as $\Delta_{pub}$ in every instance. 
Following our assumptions, let's assume two games. 

\noindent \textbf{Game 0:} In this first game, %follows the same pattern as described previously, where 
both the  challenger ($\mathcal{CHAL}$) and $\mathcal{ADV}$ will generate their respective gradient values $\Delta$ and $\tilde{\Delta}$ from their models. $\mathcal{CHAL}$ will then generate $\alpha$ and add it to $\Delta$ prior to sending $\Delta_{pub}$ back to $\mathcal{ADV}$. $\mathcal{ADV}$ will then compare two values and choose the higher one. This game can be then simplified to a functionally similar game (\textbf{Game 1}). 

\noindent \textbf{Game 1:} The core steps of the game are the same as in \textbf{Game 0}. In this case, instead of generating $\tilde{\Delta}$ and $\Delta$ from their respective models $f_{\theta_C}$, $\mathcal{ADV}$ and $\mathcal{CHAL}$ will pick a random number $x$ and $x'$ between~0 and~1, respectively. Now assuming $\alpha$ is a sufficiently large random mask, which is much higher than $x'$, the resulting $x'_{pub}$ will always be larger than $x$. And since $\mathcal{ADV}$ cannot obtain the exact value of $\alpha$ applied to the generated number $x'$, due to this she will not be able to accurately compare $x'_{pub}$ and $x$. 

The same applies in \textbf{Game 0}, as $\mathcal{ADV}$ will not be able to compare $\Delta_{pub}$ and $\tilde{\Delta}$, then $\mathcal{ADV}$ will always assume that the gradients in $\Delta_{pub}$ point to the true label in every instance, as $\Delta_{pub} > \tilde{\Delta}$.

Thus, $\tilde{f_{\theta_C}}$ does not produce similar gradients to $f_{\theta_C}$, ensuring our protocol is security against LIA.
\end{proof}

\begin{proposition}[MIA Soundness]
\label{proposition:mia-soundness}
   Let $\mathcal{ADV}$ be a semi-honest adversary that corrupts at most one of two servers ($P_0$ or $P_1$) involved in the protocol. Then $\mathcal{ADV}$ cannot launch a successful MIA.
\end{proposition}

\begin{proof}
Let $\mathcal{ADV}$ be a semi-honest adversary that has compromised $P_0$ or $P_1$. In addition to that, assume $\mathcal{ADV}$ gains access to: \begin{inparaenum}[\it (i)] \item $x_{pub} = ATm + \alpha$ \item One part of $f_{\theta_{P_{b}}}$, where $b$ is either~0 or~1, \item An architecturally similar model $\tilde{f_{\theta_C}}$ to the client's model $f_{\theta_C}$. \end{inparaenum}

% \begin{itemize}
%     \item The public input $x_{pub} = ATm + \alpha$,
%     \item One part of $f_{\theta_{P_{b}}}$, where $b$ is either~0 or~1,
%     \item An architecturally similar model $\tilde{f_{\theta_C}}$ to the client's model $f_{\theta_C}$,
% \end{itemize}

If $\mathcal{ADV}$ cannot train an attack model %, which%
that can reliably revert the %intermediate activation map 
$ATm$ from $x_{pub}$ back to its original form $x$, then our protocol is secure against MIA.

In MIA, $\mathcal{ADV}$ attempts to use the output of $f_{\theta_{C}}$ to reconstruct raw input data, in this case $\mathcal{ADV}$ relies on $x_{pub}$. Similarly, as in~\autoref{proposition:lia-soundness}, the $ATm$ produced by $f_{\theta_{C}}$ is not sent in plaintext, but is masked with a random mask $\alpha$. Due to this, $\mathcal{ADV}$ without knowing the random mask $\alpha$, she would not be able to accurately make use of her attack model as it relies on an input similar to $ATm$, but receives $x_{pub}$. Assuming $\alpha$ is securely generated following the techniques used in FSS, then $\mathcal{ADV}$ is unable to reliably revert the intermediate $ATm$ from $x_{pub}$ back to its original form $x$. Based on this we may conclude that %From this we can make a conclusion, that%
our protocol is secure against MIA.
\end{proof}

The security analysis we provide above as well as the proof against FSHA provided by the authors of MSnH~\cite{khan2024make} allows \nick{} to be robust against a variety of different individual attacks. Attacks, such as PCAT~\cite{gao2023pcat} and FORA~\cite{xu2024stealthy}, rely on similar assumptions and steps as MIA or FSHA, as such it is possible to conclude that \nick{} through the use of FSS and U-shaped SL is secure against these attack as well. This is due to the fact that \nick{} never sends the smashed data directly to $\mathcal{ADV}$, as it is hidden with a pseudo-random mask $\alpha$. This also applies to the gradients as \nick{} never provide the ADV with the gradients, but just a public input with the pseudo-random mask $\alpha$ (see \autoref{proposition:lia-soundness} and \autoref{proposition:mia-soundness}). 
\resizebox{0.8\linewidth}{!}{
\begin{myframe}{$\mathsf{GameExp}^{\mathsf{ind-fss-}\alpha}_{\nick{},\Delta}$}
\small
%\fontsize{10pt}{5pt}\selectfont
\setlength{\columnseprule}{0.3pt}
\begin{multicols}{2}
\underline{\textbf{Game 0}:}

    \underline{$\mathcal{ADV}$ computes}:
        
    $\tilde{f_{\theta_C}} \rightarrow \tilde{\Delta}$ \\
    \\
    \underline{$\mathcal{CHAL}$ computes}:\\
    $f_{\theta_C} \rightarrow \Delta$\\
    $\alpha \leftarrow \mathsf{Rand(\cdot)}$\\
    $\Delta_{pub} = \Delta + \alpha$\\
    $\mathcal{CHAL}$ sends $\Delta_{pub}$ to $\mathcal{ADV}$\\

    \underline{$\mathcal{ADV}$ compares}:\\
    $\tilde{\Delta}$ and $\Delta_{pub}$ and chooses higher one
    
    \columnbreak
    \underline{\textbf{Game 1}:}
    
    \underline{$\mathcal{ADV}$ computes}:
        
    \hl{$x \leftarrow \mathsf{Rand(\cdot)}$} \\
    \underline{$\mathcal{CHAL}$ computes}:\\
    \hl{$x' \leftarrow \mathsf{Rand(\cdot)}$};\\
    $\alpha \leftarrow \mathsf{Rand(\cdot)}$, where $\alpha >> x'$;\\
    $x'_{pub} = x' + \alpha$;\\
    $\mathcal{CHAL}$ sends $x'_{pub}$ to $\mathcal{ADV}$\\

    \underline{$\mathcal{ADV}$ compares}:\\
    $x$ and $x'_{pub}$ and chooses higher one
%\medskip
\end{multicols} 
\end{myframe}
}
\section{Performance Analysis}
\label{sec:perfanal}

\textbf{Experimental setup}: We test and compare our %private vanilla and U-shaped 
SL protocol with %the implementation proposed in 
AriaNN~\cite{ryffel2020ariann} and MSnH~\cite{khan2024make}. To %do%
make this comparison, we first reproduced the results for the AriaNN implementation~\cite{ryffel2020ariann} and MSnH~\cite{khan2024make}. For the experiments, we %make use of 
used a machine running Ubuntu 20.04 LTS, processor~12th generation Intel Core i7-12700, 32 GB RAM mesa Intel graphics (ADL-S GT1). 
We implemented our code using Python~3.7 and is made available online\footnote{\href{https://github.com/UnoriginalOrigi/SplitFSS}{https://github.com/UnoriginalOrigi/SplitFSS}} to support open and reproducible science. %and it is made %publicly
%available on %GitHub\footnote{\url{https://github.com/UnoriginalOrigi/SplitFSS}}. 
%GitHub\footnote{\url{https://anonymous.4open.science/r/SplitFSS-F74F/README.md}}.
FSS implemented using the PySyft framework\footnote{\url{https://github.com/OpenMined/PySyft/tree/ryffel/0.2.x-fix-training}} and a modified version of AriaNN\footnote{\url{https://github.com/LaRiffle/AriaNN}}. 
%We made use of the PySyft framework\footnote{\url{https://github.com/OpenMined/PySyft/tree/ryffel/0.2.x-fix-training}} for implementing FSS and based our work on a modified version of AriaNN\footnote{\url{https://github.com/LaRiffle/AriaNN}}. 
% To get a more consistent and reproducible result, we ran our experiments a total of~10 times for each experiment, totaling~60 executions %of the code 
% with different settings. 
% In the experiments, we measure the communication costs (in MB), training times (in minutes), and accuracy of the model after each epoch and report the total costs after full training. %The%
Communication costs are calculated by measuring the byte length of serialized objects sent through packets. Time measurements are taken by making use of the in-built Python \textit{time} module. %from Python. 
The loss is calculated using mean squared error. We provide the average of our experimental results in \autoref{table:performance}. Also, we have selected two models for our experiments.
\textbf{Network 1:} A 4-layered network with two Conv2D and two FC layers, utilizing Maxpool and ReLU function.
\textbf{Network 2:} It contains two Conv2D and three FC layers, utilizing Maxpool and ReLU function.
%Examples of the dataset are displayed in \autoref{fig:MNISTdata}. 

% In terms of hyperparameters, we train all networks with $E = 10$  epochs, $\eta=0.002$
% learning rate, $p = 0.9$ momentum, a batch size of $n = 128$, and the total training data samples $m = 59904$ (as the incomplete final batch is dropped). These parameters were chosen as they showed best results during the preliminary testing and are equivalent to the parameters set for AriaNN~\cite{ryffel2020ariann} and MSnH~\cite{khan2024make}. %This lets us accurately compare our results with AriaNN and MSnH without the parameters making a large impact on the comparison.  

\subsection{Evaluation}
\label{subsec:evaluation}

This section provides outcomes derived from experiments. 
% It is organized into three distinct parts, with each part dedicated to a specific architectural framework.
% as defined in \autoref{sec:architecture}.

The local models with and without SL are run in plaintext without making use of FSS to create a baseline to compare further results. They also measure the effects of FSS on the performance of the models. Since these models %does%
do not have FSS, %so%
they are represented as ``Public'' models. %As can be seen in~\autoref{table:performance}, we have public local protocol which does not have both FSS and SL, while public vanilla and public U-shaped protocols only use vanilla SL and U-shaped SL protocol. 
Second, we reproduce the results of AriaNN~\cite{ryffel2020ariann} and MSnH~\cite{khan2024make}, which we use to compare them %to%
with the public models. %are provided.% 
Also, we compare our model \nick{}, with the public U-shaped SL protocol. Then, we show the results of \nick{} alongside those of AriaNN and MSnH.
Since these protocols use FSS, they are presented as ``Private''. 

\noindent \textit{Plaintext models.} \enskip\label{subsubsec:localmodel_eval}
% As previously mentioned, training the models on plaintext data does not account for FSS. These models consist of both a local and a public SL model. We have successfully implemented and replicated the results for these models. For network~1, the best test accuracy (average of the top~3) on the MNIST dataset is 99.36\%, while for CIFAR and FMNIST datasets, the accuracies are~66.37\% and~89.44\% respectively. Similarly, for network~2, the highest test accuracy for a public local model is~99.24\% on MNIST, followed by~89.57\% on FMNIST, and the lowest is~64.65\% on CIFAR. The higher accuracy on MNIST is likely due to its simpler and more distinctive features, while the lower accuracy on CIFAR is due to its more complex and varied images. However, the important point is that the aim of this work is not to improve the accuracy of public SL or private models (AriaNN, MSnH and \nick). Instead, we are comparing how well the private models perform against the public local models. Our goal is to make the private models as accurate as the public ones.
We 
% also 
compare the public vanilla and U-shaped SL protocols to the public local model. Our experiments indicate that training the %public 
vanilla and U-shaped SL protocols on plaintext data yields similar accuracy %results 
when compared to training the public local model as illustrated in \autoref{table:performance}.
% for network~1, the public vanilla protocol achieves an accuracy of~99.26\% on MNIST, 89.88\% on FMNIST, and~66.04\% on CIFAR. Similarly, the public U-shaped SL protocol obtains an accuracy of~99.30\% on MNIST, 89.65\% on FMNIST and 65.79\% on CIFAR. 
The accuracy values %closely 
align with those achieved by the public local model, suggesting that the SL model can be effectively applied to CNN models across all three datasets without experiencing a significant degradation in %classification 
accuracy. 

We also consider the training time and communication cost of the public vanilla and U-shaped SL model and compare it to the local model. The total training time for both network~1 and network~2 across all %the %
three models %--local, vanilla SL and U-shaped SL model--
and %on all three 
datasets %--MNIST, CIFAR and FMNIST--
for~10 epochs is nearly identical, ranging between~1 and~2 minutes.
% Interestingly, even though vanilla and U-shaped SL 
% models involve communication between the client and server to exchange $ATm$ and gradients, their training times remain similar to %that of 
% the local model. This similarity arises because the local model's training time encompasses the duration during which the client %(data owner)
% transmits images to the server. Consequently, all the three models have nearly the same training times. 

%Regarding %the%
With regards to communication cost, we are considering costs on both the client and the server-side, as well as the total costs during training and testing. In \autoref{table:performance}, we show the variation between the client-side and the server-side communication costs between the public local, vanilla and the U-shaped SL protocols. We observe that the public %vanilla and 
U-shaped SL approach has slightly higher server-side communication costs than the public vanilla and local model, but client-side communications for both vanilla and U-shaped SL are about $3\times$ larger than the local model. This difference arises from the size of the data %being%
shared as in the local model, the client must send $28\times28$ images, while in the SL model, the client sends the $ATm$, which is only $8\times8$. 

In the \autoref{table:performance} presented, it is noticeable that the U-shaped SL model has slightly higher total communication costs (for both training and testing) compared to the vanilla SL models. This difference arises because, in the U-shaped SL model, there is more back-and-forth communication between the server and client involving $ATm$ and gradients, resulting in relatively higher server-side communication costs when compared to the vanilla SL models. However, the client-side communication costs show only minor variations. Even though the vanilla and U-shaped SL model provides similar accuracy and has the same training and even better communication cost when compared to the locally computed model, it is essential to note that both these models introduce a privacy leakage, %(as discussed in~\autoref{subsec:splitlear}),
which we mitigate using FSS.%. 

% \paragraph{FSS without SL} \enskip
% For our benchmarking we used AriaNN~\cite{ryffel2020ariann}--a promising solution for doing PPML on sensitive data. The authors have shown that AriaNN provides competitive performance compared to existing work, and have demonstrated its practicality by running private training on MNIST using models like LeNet. Given the promising outcomes achieved by AriaNN in comparison to other works~\cite{mohassel2018aby3}, we adopted it as the base paper for our study.

\textbf{Comparison to other work} \enskip
% Furthermore, w
We compare our work with AriaNN~\cite{ryffel2020ariann}, which secures the entire model with FSS, and %a recent approach by Khan \textit{et al.,} 
MSnH~\cite{khan2024make} where %. They introduced 
a hybrid approach combining SL and FSS was introduced. The results demonstrated that this approach addresses the privacy concern of SL while maintaining approximately the same accuracy as public local and vanilla SL models. Additionally, the authors claimed to achieve comparable accuracy to AriaNN in approximately~192 minutes, highlighting a significant advantage. 

In our experimentation, we replicated the results for AriaNN and MSnH and compared them with our proposed approach \nick{}. We tested all three models -- AriaNN, MSnH, and \nick{} -- on their ability to train NN %from scratch 
in a private manner. This end-to-end private training ensures that neither the model nor the gradients are ever accessible in plaintext. We report the training time and accuracy obtained by training %from scratch
(see~\autoref{table:performance}). While the accuracy might not match the best-known results due to the training setting, the procedure is consistent across all, ensuring a fair comparison. %, which we refer to as the private local model. 
To thoroughly assess and compare the impact of FSS on performance, we initialized two models: one before applying FSS (referred to as ``Public'') and one after applying FSS (referred to as ``Private'') (see~\autoref{table:performance}), both with the same set of initial weights. 

\textbf{Comparison of private models with the public models:} We compare each private model with its public counterpart: AriaNN with the public local model, MSnH with public vanilla SL, and \nick{} with the public U-shaped SL protocol. This comparison helps us to determine whether using FSS impacts model accuracy. Subsequently, we compare our protocol, \nick{}, with AriaNN and MSnH to evaluate its performance %relative to these%
against SotA PPML protocols.

On MNIST dataset, for network~1, our analysis, clearly indicates that the accuracy obtained before and after applying FSS remains nearly identical. %In simpler terms, the implementation of FSS exhibits no noticeable effect on the performance of the model. 
As can be seen in \autoref{table:performance}, all private models have similar accuracy to their public counterparts with approximately a~2-3\% loss in accuracy.
Our proposed protocol, \nick{}, follows this pattern, achieving a test accuracy of 97.21\%, matching its public counterpart, the U-shaped SL protocol. For the case of network~2, there is a reduction in the accuracy of %the%
private models compared to their public counterparts. More specifically, each private models accuracy falls by approximately~10\% when compared to their respective public model. Notably, \nick{} had the best performance in this case from the private models.

For the CIFAR dataset, the accuracy of private models are lower than %compared to 
their public counterparts. Specifically, in network~1, AriaNN achieves an accuracy of 65.57\%, MSnH achieves~45.46\%, and \nick{} achieves 45.78\%, whereas the SL public models demonstrate superior performance with higher accuracies by~20\%. Similarly, in network~2, private models also underperform relative to their public versions. This disparity indicates that private models experience a %considerable 
drop in accuracy on more complex datasets.

Naturally, in all benchmarks the private training time is higher than the public counterparts with increases upwards of~$900\times$. Similarly, communication costs for private models were~$2-3\times$ higher, than their public counterparts. This shows major differences in computational and communication overhead between the public and private models. The trade-off, in this case, is the improved security provided by using \nick{} or other FSS-based works.

\subsection{Experimental Results}
\label{subsec:exresul}

%This section contains \autoref{table:performance}, which documents the results of our experiments as well as notes information about the datasets used as well as the model architecture.
\begin{table*}[!h]
\centering
\caption{Training and testing results for various privacy-preserving models}
\label{table:performance}
\begin{adjustbox}{width=\textwidth}
\begin{tblr}{
  row{1} = {Gray},
  cell{1}{1} = {c=4}{},
  cell{1}{5} = {c=4}{c},
  cell{1}{9} = {c=2}{c},
  cell{3}{1} = {r=18}{},
  cell{3}{2} = {r=6}{},
  cell{3}{3} = {r=3}{},
  cell{6}{3} = {r=3}{},
  cell{9}{2} = {r=6}{},
  cell{9}{3} = {r=3}{},
  cell{12}{3} = {r=3}{},
  cell{15}{2} = {r=6}{},
  cell{15}{3} = {r=3}{},
  cell{18}{3} = {r=3}{},
  cell{21}{1} = {r=18}{},
  cell{21}{2} = {r=6}{},
  cell{21}{3} = {r=3}{},
  cell{24}{3} = {r=3}{},
  cell{27}{2} = {r=6}{},
  cell{27}{3} = {r=3}{},
  cell{30}{3} = {r=3}{},
  cell{33}{2} = {r=6}{},
  cell{33}{3} = {r=3}{},
  cell{36}{3} = {r=3}{},
  vlines,
  hline{1-3,21,39} = {-}{},
  hline{4-5,7-8,10-11,13-14,16-17,19-20,22-23,25-26,28-29,31-32,34-35,37-38} = {4-10}{},
  hline{6,12,18,24,30,36} = {3-10}{},
  hline{9,15,27,33} = {2-10}{},
}
\textbf{PETs} &  &  &  & \textbf{Training Statistics} &  &  &  & \textbf{Testing Statistics} & \\
Models & Datasets & FSS & SL & {Client \\ Comm~(MB)} & {Server \\ Comm ~(MB)} & {Training\\ Time (min)} & {Training \\ Comm~(MB)} & {Testing\\ Comm (MB)} & {Testing \\ Accuracy (\%)}\\
Network 1~ & MNIST & Public & Local & 1931.85 & \textbf{1.65} & 1.44 & 1933.51 & 313.39 & \textbf{99.36}\\
 &  &  & Vanilla & \textbf{641.17} & 25.74 & \textbf{1.42} & \textbf{666.70} & \textbf{102.52} & 99.26\\
 &  &  & U-shaped & \textbf{641.17} & 51.49 & 1.44 & 692.44 & 106.79 & 99.30\\
 &  & Private & AriaNN & 3861.32 & \textbf{3.57} & 1190.59 & 3864.90 & 1253.01 & 97.01\\
 &  &  & MSnH & \textbf{1332.08} & \textbf{3.57} & \textbf{161.49} & \textbf{1335.66} & \textbf{204.85} & 97.09\\
 &  &  & \nick & 1334.17 & 99.41 & 171.28 & 1433.59 & 205.12 & \textbf{97.21}\\
 & CIFAR & Public & Local & 4740.91 & \textbf{1.43} & 1.21 & 4742.35 & 939.60 & 66.04\\
 &  &  & Vanilla & 555.58 & \textbf{1.43} & \textbf{1.20} & \textbf{556.53} & \textbf{102.53} & \textbf{66.37}\\
 &  &  & U-shaped & \textbf{555.57} & 22.89 &\textbf{1.20} & 578.47 & 106.82 & 65.79\\
 &  & Private & AriaNN & 9479.76 & \textbf{2.97} & 2234.23 & 9483.29 & 3757.80 & \textbf{65.57} \\
 &  &  & MSnH & 2130.56 & \textbf{2.97} & \textbf{450.02} & \textbf{2133.96} & \textbf{3757.79} & 45.46\\
 &  &  & \nick & \textbf{2130.55} & 85.81 & 569.82 & 2216.80 & 3758.09 & 45.78\\
 & FMNIST      & Public   & Local & 1931.92 & \textbf{1.72} & \textbf{1.52} & 1933.65 & 313.40 & \textbf{89.88}\\
 &  &  & Vanilla & \textbf{666.70} & \textbf{1.72} & \textbf{1.52} & \textbf{668.43} & \textbf{102.50} & 89.44\\
 &  &  & U-Shaped & \textbf{666.70} & 27.47 & 1.53 & 694.17 & 106.82 & 89.65\\
 &  & Private & AriaNN & 3860.53 & \textbf{3.57} & 1814.74 & 3864.76 & \textbf{1253.01} & 74.44 \\
 &  &  & MSnH & \textbf{2556.67} & \textbf{3.57} & \textbf{354.50} & \textbf{2560.68} & \textbf{1253.01} & 77.29\\
 &  &  & \nick & 2556.68 & 102.99 & 362.36 & 2660.10 & 1253.30 & \textbf{79.23}\\
Network 2 & MNIST & Public & Local & 2507.00 & \textbf{1.72} & \textbf{1.05} & 2508.73 & 409.25 & \textbf{99.24}\\
 &  &  & Vanilla & \textbf{1011.74} & \textbf{1.72} & 1.105 & \textbf{1013.47} & \textbf{160.04} & 95.68\\
 &  &  & U-Shaped & \textbf{1011.74} & 27.47 & 1.16 & 1039.21 & 164.33 & 99.15\\
 &  & Private & AriaNN & 5010.25 & \textbf{3.52} & 2581.12 & 5014.97 & 1636.40 & 89.30\\
 &  &  & MSnH & \textbf{3936.86} & 3.57 & 753.09 & \textbf{3941.39} & \textbf{1636.39} & 85.32\\
 &  &  & \nick & 3936.87 & 102.99 & \textbf{514.08} & 4040.81 & 1636.69 & \textbf{91.98}\\
 & CIFAR & Public   & Local & 6178.61 & \textbf{1.43} & 0.91 & 6180.05 & 1227.14 & \textbf{64.65}\\
 &  &  & Vanilla & \textbf{843.11} & \textbf{1.43} & \textbf{0.90} & \textbf{844.55} & \textbf{160.04} & 63.28\\
 &  &  & U-Shaped & \textbf{843.11} & 22.89 & 0.96 & 866.01 & 164.33 & 63.70\\
 &  & Private & AriaNN & 12354.02 & \textbf{2.97} & 3343.90 & 12358.08 & \textbf{4907.95} & \textbf{63.57}\\
 &  &  & MSnH & \textbf{3280.71} & \textbf{2.97} & \textbf{843.55} & \textbf{3284.64} & \textbf{4907.95} & 44.24\\
 &  &  & \nick & \textbf{3280.71} & 85.81 & 996.12 & 3367.48 & 4908.24 & 44.07\\
 & FMNIST & Public & Local & 2507.00 & 1.72 & 1.10 & 2508.73 & 409.25 & 89.57\\
 &  &  & Vanilla & 1011.74 & 1.72 & 1.12 & 1013.47 & 160.04 & 89.22\\
 &  &  & U-Shaped & 1011.74 & 27.47 & 1.20 & 1039.21 & 164.33 & 89.58\\
 &  & Private & AriaNN & 5009.68 & \textbf{3.57} & 2466.91 & 5014.25 & 1636.40 & \textbf{86.82}\\
 &  &  & MSnH & \textbf{3937.15} & \textbf{3.57} & \textbf{498.64} & \textbf{3941.67} & \textbf{1636.39} & 77.92 \\
 &  &  & \nick & \textbf{3937.14} & 103.00 & 749.61 & 4041.10 & 1636.69 & 84.22
\end{tblr}
\end{adjustbox}
\end{table*}
To test the applicability of our approach, we used three datasets: \begin{inparaenum}[\it (1)] \item MNIST \item CIFAR \item FMNIST \end{inparaenum} on two networks: \begin{inparaenum}[\it (i)] \item Network~1: A 4-layered network with two Conv2D and two FC layers, utilizing Maxpool and ReLU function. and \item Network~2: It contains two Conv2D and three FC layers, utilizing Maxpool and ReLU function. \end{inparaenum}

In conclusion, FSS demonstrates its potential as a promising solution for maintaining data privacy while achieving commendable accuracy in ML. Nonetheless, it is important to remember that FSS comes with significant drawbacks in terms of the time it takes to train models and the extra communication needed. This should be carefully weighed when deciding if it is the right choice when selecting a privacy-preserving framework for specific use cases. Importantly, FSS solutions help reduce privacy leakages possible in the public models or in plaintext SL. The use of SL with FSS through \nick{} or MSnH additionally reduces the high computational and communication overhead of training the model fully with FSS by quite a large margain.

\textbf{Comparison of private models:} In this section, we compare \nick{} with AriaNN and MSnH. As shown in~\autoref{table:performance}, our approach achieves comparable accuracy to AriaNN and MSnH (except on CIFAR dataset) while significantly reducing training time. %, a significant advantage. 
On the MNIST dataset, \nick{} achieves a test accuracy of 97.21\% on network~1 and 91.98\% on network~2, which is comparable to the performance of both AriaNN and MSnH. Similarly, on the FMNIST dataset, \nick{} attains accuracies of~79.23 on network~1  and 84.22  on network~2, closely align with the results of AriaNN and MSnH. However, on the CIFAR dataset, while \nick{} accuracy is similar to that of MSnH, it is slightly lower than that achieved by AriaNN. These findings indicate that \nick{} performs on par with existing protocols, %in most cases, 
demonstrating its effectiveness and efficiency.

In terms of total communication cost, \nick{} exhibits slightly higher communication costs than %that of 
MSnH but remains significantly lower than %that of 
AriaNN. %Notably, \nick{} achieves substantially lower communication cost compared to AriaNN. 
However, the server-side communication cost for \nick{} is higher compared to both MSnH and AriaNN. This is because, in \nick{}, the output layer of the model is on the client-side, necessitating more communication between the client and server for sharing the $ATm$ and gradients.

Another aspect to consider is that MSnH requires minimal preprocessing communication batch (0.119 MB), representing a substantial reduction compared to AriaNN (10.572 MB per batch). This efficiency is attributed to the execution of client-side layer in plaintext, as opposed to AriaNN, which involves more interaction between the parties, increasing the communication cost. \nick{} also demonstrates ``efficiency'' in preprocessing communication costs, which are even marginally better than those of MSNH  due to few layers requiring FSS computation. Additionally, the client-side communication cost of \nick{} is~2$\times$ lower than AriaNN while remaining similar to that of MSnH. Furthermore, \nick{} demonstrates greater efficiency compared to AriaNN with a training time shorter than that of AriaNN. \nick{} operates nearly~3$\times$ more efficiently than AriaNN, %substantially 
reducing the overall training duration. The shorter training time is due to executing all client-side layers in plain, which avoids expensive techniques such as beaver triples and FSS. While \nick{} exhibits a notable improvements over AriaNN, its training time is marginally longer than that of MSnH. This comparison highlights \nick{}'s robust performance, striking a balance between efficiency and privacy. In summary, \nick{} model outperforms AriaNN in terms of training time and communication, though with an accuracy trade-off. 

A major limitation of AriaNN and MSnH is that, at the end of the protocol, both servers must communicate to get the final prediction, potentially revealing client data privacy and make it vulnerable to different attacks. In contrast,  \nick{} prevents the %model's 
output from being disclosed to $\mathcal{ADV}$ and makes it robust to the attacks mentioned above. In conclusion, \nick{} offers a promising alternative to existing protocols by achieving a trade-off between efficiency and privacy. The ability to reduce training time and communication cost, while maintaining a high level of privacy and accuracy, underscore the potential for practical application in PPML.

% \noindent \textbf{Open Science and Reproducible Research:} 
% %\label{subsec:openscience}
% To support open science and reproducible research, and provide researchers with opportunity to use, test, and extend our work, source code used for the evaluations is publicly available\footnote{\url{https://github.com/UnoriginalOrigi/SplitFSS}}.

\section{Conclusion}
\label{sec:conclusion}

Our research demonstrates that integrating SL with FSS significantly enhances the performance of FSS-based PPML. Our proposed \nick{} approach effectively addresses privacy concerns and maintains efficiency, with only a small accuracy drop compared to the public U-shaped SL protocol. In addition, \nick{} not only boosts communication efficiency and reduces complexity but also achieves the same level of accuracy compared to existing protocols like AriaNN and MSnH. Furthermore, our proposed approach expands the security of MSnH against various attacks like LIA and MIA. 

% \section*{Acknowledgment}
% This work was funded by the HARPOCRATES EU research project (No. 101069535) and the Technology Innovation Institute (TII), UAE, for the project ARROWSMITH.

\bibliographystyle{IEEEtran}
\bibliography{references}

\begin{thebibliography}{10}
\providecommand{\url}[1]{#1}
\csname url@samestyle\endcsname
\providecommand{\newblock}{\relax}
\providecommand{\bibinfo}[2]{#2}
\providecommand{\BIBentrySTDinterwordspacing}{\spaceskip=0pt\relax}
\providecommand{\BIBentryALTinterwordstretchfactor}{4}
\providecommand{\BIBentryALTinterwordspacing}{\spaceskip=\fontdimen2\font plus
\BIBentryALTinterwordstretchfactor\fontdimen3\font minus
  \fontdimen4\font\relax}
\providecommand{\BIBforeignlanguage}[2]{{%
\expandafter\ifx\csname l@#1\endcsname\relax
\typeout{** WARNING: IEEEtran.bst: No hyphenation pattern has been}%
\typeout{** loaded for the language `#1'. Using the pattern for}%
\typeout{** the default language instead.}%
\else
\language=\csname l@#1\endcsname
\fi
#2}}
\providecommand{\BIBdecl}{\relax}
\BIBdecl

\bibitem{tian2022sphinx}
H.~Tian, C.~Zeng, Z.~Ren, D.~Chai, J.~Zhang, K.~Chen, and Q.~Yang, ``Sphinx:
  Enabling privacy-preserving online learning over the cloud,'' in \emph{2022
  IEEE Symposium on Security and Privacy (SP)}.\hskip 1em plus 0.5em minus
  0.4em\relax IEEE, 2022, pp. 2487--2501.

\bibitem{hesamifard2017cryptodl}
E.~Hesamifard, H.~Takabi, and M.~Ghasemi, ``Cryptodl: Deep neural networks over
  encrypted data,'' \emph{arXiv preprint arXiv:1711.05189}, 2017.

\bibitem{sav2021poseidon}
S.~Sav, A.~Pyrgelis, J.~R. Troncoso-Pastoriza, D.~Froelicher, J.-P. Bossuat,
  J.~S. Sousa, and J.-P. Hubaux, ``Poseidon: Privacy-preserving federated
  neural network learning,'' in \emph{28th Annual Network and Distributed
  System Security Symposium, {NDSS} 2021, virtually, February 21-25,
  2021}.\hskip 1em plus 0.5em minus 0.4em\relax INTERNET SOC, 2021.

\bibitem{khan2024wildest}
T.~Khan, M.~Budzys, K.~Nguyen, and A.~Michalas, ``Sok: Wildest dreams:
  Reproducible research in privacy-preserving neural network training,'' in
  \emph{Proceedings of the 24th Privacy Enhancing Technologies Symposium
  (PETS’24)}.\hskip 1em plus 0.5em minus 0.4em\relax Berlin, Heidelberg:
  Springer-Verlag, 2024.

\bibitem{khan2023learning}
T.~Khan and A.~Michalas, ``Learning in the dark: Privacy-preserving machine
  learning using function approximation,'' in \emph{2023 IEEE 22nd
  International Conference on Trust, Security and Privacy in Computing and
  Communications (TrustCom)}.\hskip 1em plus 0.5em minus 0.4em\relax IEEE,
  2023, pp. 62--71.

\bibitem{nguyen2024pervasive}
K.~Nguyen, M.~Budzys, E.~Frimpong, T.~Khan, and A.~Michalas, ``A pervasive,
  efficient and private future: Realizing privacy-preserving machine learning
  through hybrid homomorphic encryption,'' in \emph{2024 IEEE Conference on
  Dependable, Autonomic and Secure Computing (DASC)}.\hskip 1em plus 0.5em
  minus 0.4em\relax IEEE, 2024, pp. 47--56.

\bibitem{khan2023love}
T.~Khan, K.~Nguyen, A.~Michalas, and A.~Bakas, ``Love or hate? share or split?
  privacy-preserving training using split learning and homomorphic
  encryption,'' in \emph{2023 20th Annual International Conference on Privacy,
  Security and Trust (PST)}.\hskip 1em plus 0.5em minus 0.4em\relax IEEE, 2023,
  pp. 1--7.

\bibitem{khan2023more}
T.~Khan, K.~Nguyen, and A.~Michalas, ``A more secure split: Enhancing the
  security of privacy-preserving split learning,'' in \emph{Nordic Conference
  on Secure IT Systems}.\hskip 1em plus 0.5em minus 0.4em\relax Springer, 2023,
  pp. 307--329.

\bibitem{frimpong2024guardml}
E.~Frimpong, K.~Nguyen, M.~Budzys, T.~Khan, and A.~Michalas, ``Guardml:
  Efficient privacy-preserving machine learning services through hybrid
  homomorphic encryption,'' in \emph{Proceedings of the 39th ACM/SIGAPP
  Symposium on Applied Computing}, 2024, pp. 953--962.

\bibitem{wagh2019securenn}
S.~Wagh, D.~Gupta, and N.~Chandran, ``Securenn: 3-party secure computation for
  neural network training,'' \emph{Proceedings on Privacy Enhancing
  Technologies}, 2019.

\bibitem{ryffel2020ariann}
T.~Ryffel, P.~Tholoniat, D.~Pointcheval, and F.~Bach, ``Ariann: Low-interaction
  privacy-preserving deep learning via function secret sharing,''
  \emph{Proceedings on Privacy Enhancing Technologies}, 2020.

\bibitem{wagh2021falcon}
S.~Wagh, S.~Tople, F.~Benhamouda, E.~Kushilevitz, P.~Mittal, and T.~Rabin,
  ``Falcon: Honest-majority maliciously secure framework for private deep
  learning,'' \emph{Proceedings on Privacy Enhancing Technologies}, vol. 2021,
  no.~1, pp. 188--208, 2021.

\bibitem{wagh2022pika}
S.~Wagh, ``Pika: Secure computation using function secret sharing over rings,''
  \emph{Proceedings on Privacy Enhancing Technologies}, 2022.

\bibitem{nguyen2023split}
K.~Nguyen, T.~Khan, and A.~Michalas, ``Split without a leak: Reducing privacy
  leakage in split learning,'' in \emph{Security and Privacy in Communication
  Networks (SecureComm'23), Hong Kong SAR, Hong Kong, 19—21 October
  2023}.\hskip 1em plus 0.5em minus 0.4em\relax Cham: Springer Nature
  Switzerland, 2023.

\bibitem{boyle2015function}
E.~Boyle, N.~Gilboa, and Y.~Ishai, ``Function secret sharing,'' in \emph{Annual
  international conference on the theory and applications of cryptographic
  techniques}.\hskip 1em plus 0.5em minus 0.4em\relax Springer, 2015, pp.
  337--367.

\bibitem{jawalkar2023orca}
N.~Jawalkar, K.~Gupta, A.~Basu, N.~Chandran, D.~Gupta, and R.~Sharma, ``Orca:
  Fss-based secure training and inference with gpus,'' in \emph{2024 IEEE
  Symposium on Security and Privacy (SP)}.\hskip 1em plus 0.5em minus
  0.4em\relax IEEE Computer Society, 2023, pp. 63--63.

\bibitem{abuadbba2020can}
S.~Abuadbba, K.~Kim, M.~Kim, C.~Thapa, S.~A. Camtepe, Y.~Gao, H.~Kim, and
  S.~Nepal, ``Can we use split learning on 1d cnn models for privacy preserving
  training?'' in \emph{Proceedings of the 15th ACM Asia Conference on Computer
  and Communications Security}, 2020.

\bibitem{khan2024make}
T.~Khan, M.~Budzys, and A.~Michalas, ``Make split, not hijack: Preventing
  feature-space hijacking attacks in split learning,'' in \emph{Proceedings of
  the 29th ACM on Symposium on Access Control Models and Technologies}, ser.
  SACMAT '24.\hskip 1em plus 0.5em minus 0.4em\relax New York, NY, USA:
  Association for Computing Machinery, 2024.

\bibitem{gao2023pcat}
X.~Gao and L.~Zhang, ``$\{$PCAT$\}$: Functionality and data stealing from split
  learning by $\{$Pseudo-Client$\}$ attack,'' in \emph{32nd USENIX Security
  Symposium (USENIX Security 23)}, 2023, pp. 5271--5288.

\bibitem{xu2024stealthy}
X.~Xu, M.~Yang, W.~Yi, Z.~Li, J.~Wang, H.~Hu, Y.~Zhuang, and Y.~Liu, ``A
  stealthy wrongdoer: Feature-oriented reconstruction attack against split
  learning,'' in \emph{Proceedings of the IEEE/CVF Conference on Computer
  Vision and Pattern Recognition}, 2024, pp. 12\,130--12\,139.

\bibitem{luo2023feature}
S.~Luo, F.~Yu, L.~Wang, B.~Zeng, Z.~Pang, and K.~Zhao, ``Feature sniffer: A
  stealthy inference attacks framework on split learning,'' in
  \emph{International Conference on Artificial Neural Networks}.\hskip 1em plus
  0.5em minus 0.4em\relax Springer, 2023, pp. 66--77.

\bibitem{mohassel2018aby3}
P.~Mohassel and P.~Rindal, ``Aby3: A mixed protocol framework for machine
  learning,'' in \emph{Proceedings of the 2018 ACM SIGSAC conference on
  computer and communications security}, 2018, pp. 35--52.

\bibitem{gupta2022llama}
K.~Gupta, D.~Kumaraswamy, N.~Chandran, and D.~Gupta, ``Llama: A low latency
  math library for secure inference,'' \emph{Proceedings on Privacy Enhancing
  Technologies}, vol.~4, pp. 274--294, 2022.

\bibitem{yang2023fssnn}
P.~Yang, Z.~L. Jiang, S.~Gao, J.~Zhuang, H.~Wang, J.~Fang, S.~Yiu, and Y.~Wu,
  ``Fssnn: Communication-efficient secure neural network training via function
  secret sharing,'' \emph{Cryptology ePrint Archive}, 2023.

\bibitem{pasquini2021unleashing}
D.~Pasquini, G.~Ateniese, and M.~Bernaschi, ``Unleashing the tiger: Inference
  attacks on split learning,'' in \emph{Proceedings of the 2021 ACM SIGSAC
  Conference on Computer and Communications Security}, 2021.

\bibitem{erdougan2022unsplit}
E.~Erdo{\u{g}}an, A.~K{\"u}p{\c{c}}{\"u}, and A.~E. {\c{C}}i{\c{c}}ek,
  ``Unsplit: Data-oblivious model inversion, model stealing, and label
  inference attacks against split learning,'' in \emph{Proceedings of 21st
  Workshop on Privacy in the Electronic Society}, 2022, pp. 115--124.

\bibitem{khan2023split}
T.~Khan, K.~Nguyen, and A.~Michalas, ``Split ways: Privacy-preserving training
  of encrypted data using split learning,'' in \emph{2023 Workshops of the
  EDBT/ICDT Joint Conference}.\hskip 1em plus 0.5em minus 0.4em\relax CEUR-WS,
  2023.

\bibitem{gupta2018distributed}
O.~Gupta and R.~Raskar, ``Distributed learning of deep neural network over
  multiple agents,'' \emph{Journal of Network and Computer Applications}, vol.
  116, pp. 1--8, 2018.

\bibitem{vepakomma2019reducing}
P.~Vepakomma, O.~Gupta, A.~Dubey, and R.~Raskar, ``Reducing leakage in
  distributed deep learning for sensitive health data,'' \emph{arXiv preprint
  arXiv:1812.00564}, vol.~2, p.~8, 2019.

\bibitem{mohassel2017secureml}
P.~Mohassel and Y.~Zhang, ``Secureml: A system for scalable privacy-preserving
  machine learning,'' in \emph{IEEE symposium on security and privacy
  (SP)}.\hskip 1em plus 0.5em minus 0.4em\relax IEEE, 2017, pp. 19--38.

\bibitem{li2022ressfl}
J.~Li, A.~S. Rakin, X.~Chen, Z.~He, D.~Fan, and C.~Chakrabarti, ``Ressfl: A
  resistance transfer framework for defending model inversion attack in split
  federated learning,'' in \emph{Proceedings of the IEEE/CVF Conference on
  Computer Vision and Pattern Recognition}, 2022, pp. 10\,194--10\,202.

\end{thebibliography}
\end{document}